\documentclass{article}

 \usepackage[preprint,nonatbib]{neurips_2026}

\usepackage[utf8]{inputenc} %
\usepackage[T1]{fontenc}    %
\usepackage[colorlinks=true,allcolors=blue,final]{hyperref}
\usepackage{url}            %
\usepackage{booktabs}       %
\usepackage{amsfonts}       %
\usepackage{nicefrac}       %
\usepackage{microtype}      %
\usepackage[dvipsnames,svgnames,table]{xcolor}
\usepackage{xspace}
\usepackage{graphicx}
\usepackage{siunitx}
\usepackage{tcolorbox}
\usepackage{mwe}

\usepackage[numbers]{natbib} %
\let\cite\citep

\usepackage{amsmath,amssymb,amsthm,bm,mathtools}

\usepackage[small]{caption}
\usepackage{booktabs}
\usepackage{enumitem}
\usepackage{subcaption}

\usepackage{indentfirst}
\usepackage{csquotes}
\usepackage{soul}
\usepackage{xcolor}
\usepackage{wrapfig}

\usepackage{algorithm}
\usepackage{algorithmic}

\usepackage{placeins}
\usepackage{float}
\usepackage{caption}

\newcommand{\optvar}{\bm{x}}
\newcommand{\R}{\mathbb{R}}

\newcommand{\axisbox}[1]{%
  {\fontsize{3}{3}\selectfont $#1$}%
}

\newcommand{\plotwithaxislabels}[1]{%
  \begin{tikzpicture}
    \node[inner sep=0pt, anchor=south west] (img) at (0,0)
      {\includegraphics[width=\linewidth]{#1}};

    \begin{scope}[x={(img.south east)}, y={(img.north west)}]
      \node[fill=white, draw=none, inner sep=0.5pt] at (0.47, 0.05) {\axisbox{x_0}};
      \node[fill=white, draw=none, inner sep=0.5pt, rotate=90] at (0.04,0.53) {\axisbox{x_1}};
    \end{scope}
  \end{tikzpicture}%
}
\newcommand{\axisboxtwo}[1]{%
  {\fontsize{5}{10}\selectfont $#1$}%
}
\newcommand{\plotwithaxislabelstwo}[6]{%
  \begin{tikzpicture}
    \node[inner sep=0pt, anchor=south west] (img) at (0,0)
      {\includegraphics[width=#1]{#2}};
    \begin{scope}[x={(img.south east)}, y={(img.north west)}]
      \node[fill=white, draw=none, inner sep=2pt] at (#3,#4) {\axisboxtwo{x_0}};
      \node[fill=white, draw=none, inner sep=2pt, rotate=90] at (#5,#6) {\axisboxtwo{x_1}};
    \end{scope}
  \end{tikzpicture}%
}

\DeclareMathOperator*{\argmax}{arg\,max}

\DeclareMathOperator{\Tr}{Tr}

\newtheorem{theorem}{Theorem}

\newtheorem{lemma}[theorem]{Lemma}
\newtheorem*{lemma*}{Lemma}

\newtheorem*{proposition*}{Proposition}
\newtheorem{assumption}{Assumption}

\newcommand{\ie}{i\/.\/e\/.,\/~}

\newcommand{\cf}{cf\/.\/~}

\definecolor{compare1}{RGB}{0, 114, 178}
\definecolor{compare2}{RGB}{213, 94, 0}
\definecolor{compare3}{RGB}{204, 121, 167}

\newcommand{\prefsqp}{PrefSQP\xspace}

\newcommand{\cthingy}{\bm{C}}

\usepackage{tikz}
\usetikzlibrary{arrows.meta, positioning, fit, backgrounds, calc, decorations.pathreplacing, calligraphy, shapes.geometric,tikzmark}

\DeclareRobustCommand{\markerRedBlack}{%
  \tikz[baseline=-0.5ex]\filldraw[
    fill=red,
    draw=black,
    line width=0.4pt
  ] (0,0) circle (2pt);%
}

\DeclareRobustCommand{\markerBlue}{%
  \tikz[baseline=-0.5ex]\fill[compare1] (0,0) circle (2pt);%
}
\DeclareRobustCommand{\markerBlack}{%
  \tikz[baseline=-0.5ex]\fill[black] (0,0) circle (2pt);%
}

\DeclareRobustCommand{\markerGreyCircle}{%
  \tikz[baseline=-0.5ex]\draw[gray, line width=0.5pt] (0,0) circle (2pt);%
}

\DeclareRobustCommand{\markerGreyBox}{%
  \tikz[baseline=-0.6ex]\draw[gray, line width=0.5pt] (-2pt,-2pt) rectangle (2pt,2pt);%
}

\definecolor{bayescolor}{RGB}{0,114,178}
\definecolor{sqpcolor}{RGB}{213,94,0}
\definecolor{mergecolor}{RGB}{86,180,86}

\newcommand{\highlightCt}{%
\tikz[baseline=(Ct.base)]{
  \node[
    fill=compare3!30,
    rounded corners,
    inner xsep=2pt,
    inner ysep=2pt
  ] (Ct) {$\hat{\cthingy}$};
}%
}

\tcbset{
    lavenderbox/.style={
            colback=Lavender!30,
            colframe=Lavender!80!black!50,
            coltitle=black,
            fonttitle=\bfseries,
            boxrule=0pt,
            arc=5pt,
            left=5pt,
            right=5pt,
            top=0pt,
            bottom=0pt
        },
    lavenderboxsmall/.style={
            colback=Lavender!30,
            colframe=Lavender!80!black!50,
            coltitle=black,
            fonttitle=\bfseries,
            boxrule=1pt,
            arc=5pt,
            left=5pt,
            right=5pt,
            top=0pt,
            bottom=0pt
        },
    bluebox/.style={
            colback=LightBlue!30,
            colframe=LightBlue!80!black,
            coltitle=black,
            fonttitle=\bfseries,
            boxrule=1pt,
            arc=5pt,
            left=5pt,
            right=5pt,
            top=5pt,
            bottom=5pt
        },
    greenbox/.style={
            colback=Green!15,
            colframe=Green!60!black,
            coltitle=black,
            fonttitle=\bfseries,
            boxrule=0.8pt,
            arc=5pt,
            left=5pt,
            right=5pt,
            top=5pt,
            bottom=5pt
        },
    orangebox/.style={
            colback=Orange!20,
            colframe=Orange!90!black,
            coltitle=black,
            fonttitle=\bfseries,
            boxrule=1pt,
            arc=5pt,
            left=5pt,
            right=5pt,
            top=5pt,
            bottom=5pt
        },
    yellowbox/.style={
            colback=Yellow!20,
            colframe=Yellow!50!black,
            coltitle=black,
            fonttitle=\bfseries,
            boxrule=0.8pt,
            arc=5pt,
            left=5pt,
            right=5pt,
            top=5pt,
            bottom=5pt
        }
}

\title{Local Preferential Bayesian Optimization}

\author{
Johanna Menn$^{1,*}$ \quad
Miriam Kober$^{2,3,*}$ \quad
Paul Brunzema$^{1}$ \quad
David Stenger$^{4}$ \quad
Sebastian Trimpe$^{1}$ \\
\\
$^{1}$Institute for Data Science in Mechanical Engineering, RWTH Aachen University, Aachen, Germany \\
$^{2}$Department of Clinical Research, University of Bern, Bern, Switzerland \\
$^{3}$Center for Reproducible Science and Research Synthesis, University of Zurich, Zurich, Switzerland \\
$^{4}$aiXopt GmbH, Aachen, Germany \\
$^{1}$\texttt{\{johanna.menn, paul.brunzema, trimpe\}@dsme.rwth-aachen.de}\\
$^{2}$\texttt{miriam.kober@unibe.ch}, $^{4}$\texttt{dstenger@aixopt.ai} 
\\
$^{*}$Equal contribution
}

\begin{document}

\maketitle

\begin{abstract}
    Bayesian optimization (BO) is a popular and effective approach for tuning expensive, noisy experiments, but requires the formulation of an explicit objective function.
    Preferential BO (PBO) removes this requirement by learning from pairwise human feedback, yet existing methods struggle to efficiently optimize beyond low- and medium-dimensional problems due to their global search approaches.
    We address this limitation by developing a family of local PBO methods that transfer key ideas from high-dimensional BO to the preferential setting.
    In particular, we introduce
    local PBO methods
    which adapt trust-region and derivative-informed local search to pairwise preference feedback, where the latter exploits first- and second-order derivatives of the Laplace-approximated GP posterior.
    Our benchmark on GP sample paths, standard optimization benchmark functions, and policy-search tasks shows that local PBO methods are especially effective in high-dimensional and complex landscapes with steep optima.
    Compared with global preference-based baselines, they can substantially reduce cumulative regret, making them particularly useful for real-world preference-based optimization tasks such as policy search.
\end{abstract}

\section{Introduction}

Bayesian optimization (BO) \cite{garnett_bayesoptbook_2023} is widely used to tune expensive black-box systems, including controller parameters \cite{neumann2019data, coutinho2024human, holzapfel2024event, menn2024lipschitz}, industrial processes \cite{deblasi2024safe, zimmermann2025bayesian}, and policies for robotic systems \cite{he2025simulation, hose2024fine}.
However, many such applications do not admit a natural scalar objective: performance may depend on qualitative user feedback based on criteria such as smoothness or comfort.
Preferential Bayesian optimization (PBO) addresses this issue by replacing numeric objective evaluations with pairwise comparisons \cite{gonzalez2017PBO}, making it particularly well suited to real-world optimization tasks and human feedback.

Despite its practical appeal, existing PBO methods have primarily been developed and applied in low- and medium-dimensional settings \cite{astudillo2023qEUBO, Xu2024PrincipledPB, gonzalez2017PBO}. 
In high-dimensional problems, the global search strategies of these state-of-the-art methods can become inefficient due to the rapidly growing search space \cite{muller2021local,wu2023behavior}.
Standard BO often addresses such settings through local search approaches, such as trust-region BO methods \cite{eriksson2019scalable} and gradient-based BO methods \cite{muller2021local,nguyen2022local,tang2025nest,he2025simulation,brunzema_bayesqp_2026}.
This is especially attractive when cumulative regret matters in addition to final performance: in real-world applications such as controller tuning, poor intermediate evaluations can be costly \cite{menn2026preferential}.

In this paper, we revisit preferential Bayesian optimization through the lens of local search.
Rather than treating PBO as a purely global search problem, we develop a local view in which pairwise comparisons are used to progressively refine promising regions of the design space.
Building on this perspective, we adapt two central ideas from high-dimensional BO, namely trust-region optimization~\cite{eriksson2019scalable} and derivative-informed~\cite{muller2021local, brunzema_bayesqp_2026} local refinement, to the preferential setting (Figure~\ref{fig:search_mechanism}).
The resulting algorithms form a new family of local PBO methods.
Among them, we introduce, to the best of our knowledge, the first PBO methods that leverage higher-order information from the Laplace-approximated Gaussian process (GP) posterior, including both first- and second-order variants.

Our experiments on GP sample paths, standard optimization benchmark functions, and policy-search tasks show that the proposed local PBO methods are particularly effective in high-dimensional and complex landscapes with steep optima.
While global methods remain competitive in smoother and lower-dimensional settings, local approaches often achieve lower cumulative regret by efficiently refining promising regions.
These results make local PBO methods especially useful for real-world preference-based optimization tasks.

\paragraph{Contributions}
In summary, the key contributions of this paper are:
\begin{enumerate}[label=(\roman*), itemsep=0pt, topsep=0pt, parsep=0pt, partopsep=0pt, leftmargin=*]
    \item[\textbf{C1}] We introduce a new family of local PBO methods that bring trust-region optimization and derivative-informed local refinement to pairwise preference-based optimization.
    \item[\textbf{C2}] With GIPBO and PrefSQP, we develop the first derivative-informed PBO methods, which exploit first- and second-order posterior predictive information from the Laplace-approximated pairwise GP posterior.
    \item[\textbf{C3}] We benchmark local PBO against global preference-based baselines on GP sample paths, standard benchmark functions, and reinforcement-learning policy-search tasks, showing that local methods are especially effective in complex high-dimensional landscapes and can reduce cumulative regret when good initial regions are available.
\end{enumerate}

\begin{figure}[t]
    \centering

    \begin{subfigure}{0.32\linewidth}
        \centering
        \includegraphics[width=1\linewidth, trim=0 0 0 0.0cm, clip]{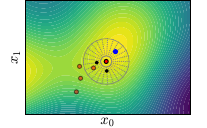}
        \caption{GIPBO}
        \label{fig:placeholder1}
    \end{subfigure}
    \hfill
    \begin{subfigure}{0.32\linewidth}
        \centering
        \includegraphics[width=1\linewidth, trim=0 0 0 0.0cm, clip]{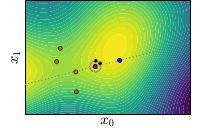}
        \caption{PrefSQP}
        \label{fig:placeholder2}
    \end{subfigure}
    \hfill
    \begin{subfigure}{0.32\linewidth}
        \centering
        \includegraphics[width=1\linewidth, trim=0 0 0 0.0cm, clip]{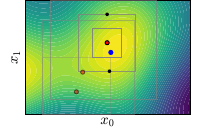}
        \caption{TuRPBO}
        \label{fig:placeholder2}
    \end{subfigure}
    \caption{Overview of the local methods with 15 evaluations per plot. In GIPBO and PrefSQP, \markerRedBlack{} marks the gradient-estimation center; in TuRPBO, it marks the trust-region center. \markerBlack{} marks previous evaluations, \markerBlue{} the next evaluation, \markerGreyCircle{} search bounds, and \markerGreyBox{} the trust region.
(a) GIPBO selects points within the search bounds to minimize gradient uncertainty at the center, then takes a fixed-step gradient update.
(b) PrefSQP estimates the gradient and Hessian by random search within the bounds, followed by a line search in the descent direction.
(c)~TuRPBO adapts the trust-region size heuristically.}
    \label{fig:search_mechanism}
    \vspace{-1em}
\end{figure}

\section{Problem Formulation}

We consider the maximization of an unknown function \(f: \Theta \rightarrow \mathbb{R}\) over a compact set \(\Theta \subseteq \mathbb{R}^d\):
\begin{equation}
    \optvar^* \in \arg\max_{\optvar \in \Theta} f(\optvar).
    \label{eq:optimization problem}
\end{equation}
Instead, at each iteration \(n \in \mathbb{I}_T \coloneqq \{1, \dots, T\}\), where \(T\) is the total query budget, an algorithm queries a pair of inputs
\((\optvar_n,\optvar_n') \in \Theta^2\) and observes only a binary preference.
Specifically, the underlying noisy evaluations of the unknown utility function are given by
\begin{equation}
    y_n = f(\optvar_n) + \epsilon_n,
    \qquad
    y_n' = f(\optvar_n') + \epsilon_n',
    \qquad
    \epsilon_n,\epsilon_n' \overset{\mathrm{i.i.d.}}{\sim}
    \mathcal{N}(0,\sigma_\mathrm{n}^2).
    \label{eq:noise}
\end{equation}
The observation is a binary random variable \(\pi_n \in \{0,1\}\) indicating which noisy evaluation is larger:
\begin{equation}
    \pi_n =
    \begin{cases}
        0, & \text{if } y_n > y_n',\\
        1, & \text{otherwise}.
    \end{cases}
    \label{eq:pi}
\end{equation}
Equivalently, each observation can be represented as an ordered comparison \((\optvar_n^+,\optvar_n^-)\), where \(\optvar_n^+\) denotes the preferred input and \(\optvar_n^-\) the non-preferred input.

\paragraph{Goal of the Paper}
We aim to solve \eqref{eq:optimization problem} in high-dimensional settings using only noisy pairwise comparisons \eqref{eq:pi}.
To this end, we propose local PBO methods that augment PBO with state-of-the-art local BO mechanisms. Specifically, we characterize the changes needed to enable the use of these mechanisms with pairwise observations and to evaluate when local PBO improves over existing global preference-based optimization methods.

\section{Related Work}

\paragraph{Preferential Bayesian Optimization}
Existing PBO methods are primarily built on the pairwise GP model of \cite{chu2005preference}, combined with global acquisition strategies.
Early work pairs this surrogate with Thompson sampling (TS) \cite{gonzalez2017PBO}, as implemented in the Polar framework \cite{tucker2022polar}. More recent methods use the decision-theoretic acquisition function EUBO \cite{lin2022preference} and its batch extension qEUBO \cite{astudillo2023qEUBO}, which outperform TS approaches and are available in BoTorch \cite{balandat2020botorch}. 
POP-BO \cite{Xu2024PrincipledPB} is, to our knowledge, the only PBO method with regret guarantees, but relies on direct function evaluations for hyperparameter selection and is therefore not applicable in a strictly preference-based setting \cite{erarslan_consecutive_2025}.
Beyond pairwise GPs, the skewed GP approximation \cite{benavoli2021preferential} underlies the hallucination believer method \cite{takeno2023towards}, which can be used with standard acquisition functions such as UCB and EI, while parametric, non-Bayesian methods like GLISp \cite{bemporad2021global} model the utility as a linear combination of radial basis functions.
Despite these advances, all state-of-the-art PBO methods rely on global search and therefore inherit the same inefficiencies as standard BO in high dimensions.

\paragraph{Scaling PBO to High Dimensions}
High-dimensional PBO has mainly been addressed through subspace reduction: LineCospar \cite{trucker2020human} uses the mechanism of \citet{kirschner2019adaptive} to scale up to six dimensions, \citet{zhang2023high} exploit the assumption that effective dimensionality is lower than the ambient one, and \citet{mikkola2020projective} similarly employ projective subspaces.
These methods all assume the preference structure lies in a low-dimensional subspace, which may not hold in complex settings.
An alternative direction is slider-based PBO \cite{koide2024high}, which balances local exploitation and global exploration via EUBO, but its feedback mechanism may not generalize and no implementation is publicly available. 
Overall, scalable PBO methods that operate directly in the original high-dimensional space remain limited; this paper proposes new ideas for scaling PBO.

\paragraph{Local Bayesian Optimization}
While subspace projections have also been explored in standard BO \cite{kirschner2019adaptive}, restricting the search to local subregions has emerged as a key mechanism for high-dimensional problems, achieving substantially improved sample efficiency by prioritizing local refinement over global exploration.
Prominent approaches follow two paradigms: trust-region methods \cite{eriksson2019scalable, eriksson_scalable_2021,yi2024improving} that iteratively restrict the search to dynamically adapted local regions, and gradient-informed methods such as GIBO \cite{muller2021local} and subsequent variants \citep{nguyen2022local, tang_cages_2024, he2025simulation,fan_minimizing_2024, von_rohr_local_2024, ip2025user} that exploit local gradient information of the surrogate.
More recently, BayeSQP \cite{brunzema_bayesqp_2026} extends the gradient-based line by additionally incorporating second-order information.
Motivated by the promise of these local approaches and the lack of scalable PBO methods, we adapt trust-region and derivative-informed local BO mechanisms to the preferential setting, giving rise to a new family of local PBO methods.

\section{Preference Modeling with Pairwise Gaussian Processes}
\label{sec:pref_modeling}

Our local PBO methods use pairwise GPs as surrogate models, which are the standard modeling choice in many PBO methods \cite{gonzalez2017PBO, tucker2022polar, lin2022preference, astudillo2023qEUBO, Xu2024PrincipledPB}.
In this section, we first review pairwise GPs following \cite{chu2005preference} and \cite{brochu2010tutorial}.
We then extend this formulation by deriving the approximate joint conditional distribution of the latent function value, its gradient, and its Hessian under the Laplace-approximated pairwise GP posterior.
This joint distribution provides the predictive means and variances needed by our derivative-informed local PBO methods GIPBO and PrefSQP.
Since this is the first work to use first- and second-order posterior predictive information in PBO, we explicitly formalize these quantities for pairwise GPs and state the assumptions required for the derivations.

\paragraph{Prior}
GPs \cite{rasmussen2006gps} model functions over arbitrary inputs and are characterized by their mean and covariance functions.
In the preferential setting, we follow \cite{chu2005preference} and \cite{brochu2010tutorial}, and assume that the latent function values are modeled by a zero-mean GP without loss of generality.
Let \( \bm{X} = [\optvar_i]_{i=1, \dots, P}\) denote the set of unique input locations appearing in the pairwise comparisons.
Evaluating the GP at these locations yields a multivariate Gaussian random variable \(\bm{f} = [f(\optvar_1), f(\optvar_2), ..., f(\optvar_M)]^\top\) with zero mean and covariance matrix \(K(\bm{X},\bm{X})\in \mathbb{R}^{P \times P}\), whose \(ij\)-th entry is determined by the covariance function \(k(\optvar_i, \optvar_j)\).
The GP prior is then $\bm{f} \sim \mathcal{N} (0, K(\bm{X},\bm{X}))$.
In the following, we will use the shorthand notation of $\bm{K}_{\bm{XX}} \coloneqq K(\bm{X},\bm{X})$ for notational convenience.

\paragraph{Likelihood} Unlike standard BO, PBO does not observe noisy function values directly. Instead, each query yields only a pairwise preference: recall that given two inputs \(\optvar_n\) and \(\optvar^\prime_n\), we observe which output \(y_n = f(\optvar_n) + \epsilon_n\) or \(y^\prime_n = f(\optvar^\prime_n) + \epsilon^\prime_n\), with \(\epsilon_n,\epsilon^\prime_n \sim \mathcal{N}(0, \sigma_\mathrm{n}^2)\), is larger. We write each comparison as an ordered pair \((\optvar_n^+,\optvar_n^-)\), where \(\optvar_n^+\) denotes the preferred input and \(\optvar_n^-\) the non-preferred input. This induces the probit likelihood
\begin{equation}
p(\optvar_n^+ \succ \optvar_n^- \mid \bm{f})
=
p(y_n^+ > y_n^- \mid f(\optvar_n^+), f(\optvar_n^-))
=
\Phi(z_n),
\qquad
z_n =
\frac{f(\optvar_n^+) - f(\optvar_n^-)}{\sqrt{2}\sigma_\mathrm{n}}.
\end{equation}
Here, \(\Phi\) denotes the cumulative distribution function of the standard normal.
For an ordered data set of pairwise comparisons \(\mathcal{D}= \{(\optvar_n^+,\optvar_n^-)\}_{n=1}^N\), with \(N\) observations, the likelihood is then the factorization
\begin{equation}
\textstyle
    p(\mathcal{D} \mid \bm{f}) = \prod_{n=1}^N \Phi(z_n).
    \label{eq:likelihood}
\end{equation}
\paragraph{Posterior}
Since the likelihood in \eqref{eq:likelihood} is non-Gaussian, the posterior \(P(\bm{f} \mid \mathcal{D})\propto P( \mathcal{D}\mid\bm{f})P(\bm{f})\) is also non-Gaussian.
We therefore need to find an approximation of the posterior and follow \citet{chu2005preference} who use a Laplace approximation around the maximum a posteriori (MAP) estimate
\begin{equation}
    \hat{\bm{f}}
    =
    \arg\min_{\bm{f}} S(\bm{f}),
    \qquad \text{where} \qquad
    S(\bm{f})
    =
    -\sum_{n=1}^{N}\log\Phi(z_n)
    +
    \frac{1}{2}\bm{f}^\top \bm{K}_{\bm{XX}}^{-1}\bm{f} .
\end{equation}
The optimality condition and Hessian at the MAP are
\begin{equation}
    \nabla S(\hat{\bm{f}})
    =
    -\nabla_{\bm{f}}\log p(\mathcal{D}\mid\bm{f})\big|_{\hat{\bm{f}}}
    +
    \bm{K}_{\bm{XX}}^{-1}\hat{\bm{f}}
    =
    0,
    \qquad
    \nabla^2 S(\hat{\bm{f}})
    =
    \bm{K}_{\bm{XX}}^{-1}+\cthingy .
\end{equation}

Let \(\bm{h}_n\in \R^M\) have entries \(+1\) at \(\optvar_n^+\), \(-1\) at \(\optvar_n^-\), and \(0\) otherwise, so that \(z_n=\bm{h}_n^\top\bm{f}/(\sqrt{2}\sigma_\mathrm{n})\). Then
\begin{equation}
    \cthingy
    =
    \sum_{n=1}^{N}
    \frac{1}{2\sigma_\mathrm{n}^2}
    \left[
        \left(\frac{\varphi(z_n)}{\Phi(z_n)}\right)^2
        +
        z_n\frac{\varphi(z_n)}{\Phi(z_n)}
    \right]_{\hat{\bm{f}}}
    \bm{h}_n \bm{h}_n^\top .
    \label{eq:C}
\end{equation}
where \(\varphi\) is the probability density function of the standard normal distribution.
With this, we can then formulate the Laplace-approximated posterior distribution as
\begin{equation}
    p(\bm{f}\mid \mathcal{D})
    \approx
    q(\bm{f}\mid \mathcal{D})
    =
    \mathcal N\!\left(
        \hat{\bm{f}},
        \left(\bm{K}_{\bm{XX}}^{-1}+\cthingy\right)^{-1}
    \right).    
\end{equation}
Crucially, \(\cthingy\in\mathbb{R}^{P\times P}\) is data-dependent through the MAP estimate \(\hat{\bm f}\) and the comparison outcomes.
This differs from standard GP regression with Gaussian observation noise, where the corresponding likelihood precision contribution is typically a fixed matrix, such as \(\sigma_{\text{noise}}^{-2}\bm{I}\).
This will be important when actively trying to minimize gradient uncertainty in Sec.~\ref{sec:gipbo}.

\paragraph{Joint Conditional Distribution}
To enable derivative-informed optimization of the objective function, analogous to local BO methods in the standard scalar-feedback setting \cite{muller2021local, brunzema_bayesqp_2026}, we require not only the predictive mean and variance of \(f\), but also posterior information about its gradient and Hessian.
For a sufficiently differentiable GP prior, function values and derivatives are jointly Gaussian \cite{rasmussen2006gps}.
Since the pairwise preference likelihood makes the exact posterior non-Gaussian, we use the Laplace approximation introduced above, which yields a Gaussian approximate surrogate.
This induces the following approximate joint predictive distribution at a test point \(\optvar\):
\begin{equation}
    \left(
    \begin{bmatrix}
        f(\optvar) \\
        \nabla f(\optvar) \\
        \operatorname{vec} \{\nabla^2 f(\optvar)\}
    \end{bmatrix}
    \;\middle|\;
    \bm{X},\mathcal{D}
    \right)
    \approx
    \mathcal{N}\!\left(
    \begin{bmatrix}
        \mu_f^{\mathrm{pref}}(\optvar) \\
        \boldsymbol{\mu}_{\nabla f}^{\mathrm{pref}}(\optvar) \\
        \operatorname{vec}  \{\boldsymbol{\mu}_{\nabla^2 f}^{\mathrm{pref}}(\optvar) \}
    \end{bmatrix},
    \bm{\Sigma}_{\bm{g}}^{\mathrm{pref}}(\optvar)
    \right),
\label{eq:joint_derivative_predictive}
\end{equation}
where \(\bm{\Sigma}_{\bm{g}}^{\mathrm{pref}}(\optvar)\) contains the covariance blocks between \(f(\optvar)\), \(\nabla f(\optvar)\), and \(\operatorname{vec} \{\nabla^2 f(\optvar)\}\).
The induced approximate predictive means are
\begin{equation}
    \resizebox{\textwidth}{!}{$
    \underbracket{
    \mu_f^{\mathrm{pref}}(\optvar)
    =
    K(\optvar,\bm{X})\bm{K}_{\bm{XX}}^{-1}\hat{\bm{f}}
    }_{\in \mathbb{R}}
    \quad
    \underbracket{
    \boldsymbol{\mu}_{\nabla f}^{\mathrm{pref}}(\optvar)
    =
    \nabla_{\optvar} K(\optvar,\bm{X})\bm{K}_{\bm{XX}}^{-1}\hat{\bm{f}}
    }_{\in \mathbb{R}^{d}}
    \quad
    \underbracket{
    \boldsymbol{\mu}_{\nabla^2 f}^{\mathrm{pref}}(\optvar)
    =
    \nabla_{\optvar}^2 K(\optvar,\bm{X})\bm{K}_{\bm{XX}}^{-1}\hat{\bm{f}}
    }_{\in \mathbb{R}^{d\times d}}
    $} \nonumber
\end{equation}
Since $\cthingy$ is only guaranteed to be positive semidefinite \cite{chu2005preference}, we use the regularized matrix $\hat{\cthingy}=\cthingy+\varepsilon_{\mathrm{reg}}\bm{I}$,
where $\varepsilon_{\mathrm{reg}}>0$ is a small regularization constant.
For compactness, let
\(\mathcal{L}_{0,\optvar} f=f(\optvar)\),
\(\mathcal{L}_{1,\optvar} f=\nabla f(\optvar)\), and
\(\mathcal{L}_{2,\optvar} f=\operatorname{vec}\{\nabla^2 f(\optvar)\}\), and define
\(\bm{k}_{a}(\optvar,\bm{X})=\mathcal{L}_{a,\optvar}K(\optvar,\bm{X})\) for \(a\in\{0,1,2\}\).
Then, under the Laplace-approximated pairwise GP posterior, the covariance block between derivative orders \(a,b\in\{0,1,2\}\) is
\begin{equation}
    \bm{\Sigma}_{ab}^{\mathrm{pref}}(\optvar)
    =
    \mathcal{L}_{a,\optvar}\mathcal{L}_{b,\optvar'} 
    k(\optvar,\optvar')\big|_{\optvar'=\optvar}
    -
    \bm{k}_{a}(\optvar,\bm{X})
    \bigl(
        \bm{K}_{\bm{XX}} + \hat{\cthingy}^{-1}
    \bigr)^{-1}
    \bm{k}_{b}(\bm{X},\optvar),
    \label{eq:general_covariance_block}
\end{equation}
where \(\bm{k}_{b}(\bm{X},\optvar)=\bm{k}_{b}(\optvar,\bm{X})^\top\), with the corresponding derivatives taken with respect to the second kernel argument.
The full covariance matrix in \eqref{eq:joint_derivative_predictive} is obtained by stacking these blocks as
\begin{equation}
    \bm{\Sigma}_{\bm{g}}^{\mathrm{pref}}(\optvar)
    =
    \begin{bmatrix}
        \bm{\Sigma}_{00}^{\mathrm{pref}}(\optvar)
        &
        \bm{\Sigma}_{01}^{\mathrm{pref}}(\optvar)
        &
        \bm{\Sigma}_{02}^{\mathrm{pref}}(\optvar)
        \\
        \bm{\Sigma}_{10}^{\mathrm{pref}}(\optvar)
        &
        \bm{\Sigma}_{11}^{\mathrm{pref}}(\optvar)
        &
        \bm{\Sigma}_{12}^{\mathrm{pref}}(\optvar)
        \\
        \bm{\Sigma}_{20}^{\mathrm{pref}}(\optvar)
        &
        \bm{\Sigma}_{21}^{\mathrm{pref}}(\optvar)
        &
        \bm{\Sigma}_{22}^{\mathrm{pref}}(\optvar)
    \end{bmatrix}.
    \label{eq:full_cov}
\end{equation}
We will now build our family of local PBO algorithms based on this joint predictive formulation. %

\section{Local Search Mechanisms for Preferential Optimization}
\label{sec:method}

Having revisited the pairwise GP framework and derived its predictive distributions, we now address the central question of this paper:
\textit{how can we leverage local BO mechanisms to efficiently solve high-dimensional preferential optimization problems?}
For this, we transfer three established local BO strategies to the preference setting.
Each represents a different approach to local search in black-box optimization.
For each method, we first discuss the standard BO algorithm it builds upon and then detail the adaptations required for pairwise feedback.

\subsection{GIPBO: Gradient-based Preferential Bayesian Optimization}
\label{sec:gipbo}

The first class of local BO algorithms that we reinterpret for the local PBO setting is gradient-informed search.
Our algorithm GIPBO extends GIBO \cite{muller2021local}, which uses the GP surrogate to actively learn the gradients and performs first-order gradient steps,
to pairwise preferential feedback.

\paragraph{Gradient Information with Bayesian Optimization}
GIBO \cite{muller2021local} and its variants \citep{nguyen2022local,tang_cages_2024, he2025simulation,fan_minimizing_2024, von_rohr_local_2024, ip2025user} leverage the fact that the derivative of a GP is also a GP \cite{rasmussen2006gps}.
At iteration $t$, GIBO selects $M = d$ new query points to estimate the gradient at a current vector $\optvar_t$ by maximizing the expected reduction in uncertainty of the GP's Jacobian via the Gradient Information (GI) acquisition function:
\begin{equation}
  \begin{aligned}
    \alpha_{\mathrm{GI}}(\optvar\mid \optvar_t, \mathcal{D}) = \mathbb{E}\Bigl[\,
         \Tr\bigl(\bm{\Sigma}_{11}(\optvar_t\mid \mathcal{D})\bigr) - \Tr\bigl(\bm{\Sigma}_{11}(\optvar_t\mid \{\mathcal{D},(\optvar,y)\})\bigr),
      \Bigr].
  \end{aligned}
  \label{eq:GI_acf}
\end{equation}
where $\bm{\Sigma}_{11}$ is the Jacobian covariance analogous to \eqref{eq:full_cov}.
Because a GP's covariance is independent of observed values, this criterion reduces to comparing Jacobian variances before and after adding $\optvar$ to the dataset.
$\bm{\Sigma}_{11}$ thus only depends on the virtual dataset $\hat{\bm{X}} = [\bm{X}, \optvar]$, where $X \subset \mathbf{\Theta}$. At $\optvar_t$, it is
\begin{equation}
       \bm{\Sigma}_{11}(\optvar_t \mid [\bm{X}, \optvar]) = \nabla^2_{\optvar_t} K(\optvar_t, \optvar_t)
- \nabla_{\optvar_t} K(\optvar_t, \hat{\bm{X}}) 
(\bm{K}_{\hat{\bm{X}}\hat{\bm{X}}} + \sigma_{\text{noise}}^2 I)^{-1} \nabla_{\optvar_t} K(\hat{\bm{X}}, \optvar_t). 
\label{eq:Jac_variance}
\end{equation}
After collecting $M$ such points, the GP posterior of $\nabla f$ is updated and a gradient step $\optvar_{t+1} = \optvar_t + \eta \cdot \bm{p}_t$ with step length $\eta$ and with $\bm{p}_t=\mathbb{E}[\nabla f(\optvar_t)]$ is performed. %

\paragraph{Extension to Preferential Feedback}
We now extend GIBO to pairwise preferential feedback. %
Given two initial samples, we take the preferred one as the starting point.
Each iteration repeats: \textit{(i)}~estimate the gradient at the current point via $M$ experiments; \textit{(ii)}~move one step along this estimate.
We repeat these steps until the sample budget is empty.
The further discussion and the full algorithm are presented in Appendix~\ref{sec:GIPBO_implementation_details} and Algorithm~\ref{alg:gipbo}.

Naively reusing the GI acquisition function \eqref{eq:GI_acf} for PBO fails because the Jacobian covariance of the pairwise GP $\bm{\Sigma}_{11}^{\mathrm{pref}}$ at $\optvar_t$ no longer depends solely on the virtual dataset $\hat{\bm{X}}$. Instead, the pairwise GP's Jacobian variance following \eqref{eq:general_covariance_block} is
\begin{equation}
\label{eq:PairwiseGP_Jac_variance}
\bm{\Sigma}_{11}^{\mathrm{pref}} = \nabla^2_{\optvar_t}K(\optvar_t,\optvar_t)
- \nabla_{\optvar_t}K(\optvar_t,\hat{\bm{X}})\,
\bigl(\bm{K}_{\hat{\bm{X}}\hat{\bm{X}}}+ \highlightCt^{-1}\bigr)^{-1}
\,\nabla_{\optvar_t}K(\hat{\bm{X}},\optvar_t)\,.
\end{equation}
where the matrix $\highlightCt$ would now depend on the outcome of a duel $(\optvar_t, \optvar)$.
This forces us to consider two cases; the virtual data point being preferred over the current parameter, or vice versa; and take a conditional expectation based on the surrogate model, computed via \eqref{eq:likelihood}.
The preferential gradient-information acquisition function is therefore
\begin{equation}
    \begin{aligned}
        \alpha_{\mathrm{GI\text{-}PBO}}(\optvar \mid \optvar_t, \mathcal{D}) 
        =  \Tr \big( \bm{\Sigma}_{11}^{\mathrm{pref}}(\optvar_t \mid \mathcal{D}) \big) 
        &- P(\textcolor{compare1}{\optvar_t \succ \optvar}) 
        \Tr \big( 
            \bm{\Sigma}_{11}^{\mathrm{pref}}
            (\optvar_t \mid \hat{\mathcal{D}}_{\textcolor{compare1}{\optvar_t \succ \optvar}}) 
        \big) \\
        &- P(\textcolor{compare2}{\optvar \succ \optvar_t})
        \Tr \big( 
            \bm{\Sigma}_{11}^{\mathrm{pref}}
            (\optvar_t \mid \hat{\mathcal{D}}_{\textcolor{compare2}{\optvar \succ \optvar_t}})
        \big)
    \end{aligned}
    \label{eq:GI_PBO_acf_detailed}
\end{equation}
where $\hat{\mathcal{D}}_{\textcolor{compare1}{\optvar_t \succ \optvar}} := \{\mathcal{D}, ((\optvar_t, \optvar), \hat\pi = 0)\}$ denotes the virtual dataset for $\optvar_t$ being preferred over $\optvar$ and $\hat{\mathcal{D}}_{\textcolor{compare2}{\optvar \succ \optvar_t}}$  the opposite case.
To obtain the next query $\optvar_{i+1}$, we maximize \eqref{eq:GI_PBO_acf_detailed}.
Analogous to GIBO, we optimize the acquisition function inside search bounds around the current $\optvar_t$.
To account for discernibility limits \cite{luce1956semiorders}, we also use a lower search bound.
Due to the dependence on $\cthingy$, we have to apply a grid-based optimization strategy by selecting points from various $d$-dimensional hyperspheres around the current parameters $\optvar_t$.
Each virtual comparison for each of the grid points requires refitting the model, which makes the acquisition function very expensive.
For high-dimensional settings, we therefore use GIPBOr, where points inside these search bounds are suggested randomly.

During gradient estimation, GIPBO maintains a local incumbent initialized at the current center point, $\optvar_{t,0}^{\star}=\optvar_t$. 
For the $i$-th local query $\optvar_{t,i}$, the method observes a preference between $\optvar_{t,i}$ and the current incumbent $\optvar_{t,i-1}^{\star}$, then updates
\begin{equation}
    \optvar_{t,i}^{\star}
    =
    \begin{cases}
        \optvar_{t,i}, & \text{if } \optvar_{t,i} \succ \optvar_{t,i-1}^{\star}, \\
        \optvar_{t,i-1}^{\star}, & \text{otherwise}.
    \end{cases}
\end{equation}
The preferential GP is updated after each duel, and the acquisition process is repeated for $M=d$ local queries. GIPBO then estimates the local gradient direction and takes a step along $\bm{p}_t = \boldsymbol{\mu}_{\nabla f}^{\mathrm{pref}}(\optvar_t)$ to the next center point $\optvar_{t+1}$, without performing an additional preference comparison for this step. The next iteration starts a new gradient-estimation phase centered at $\optvar_{t+1}$.

\subsection{PrefSQP: Preference-based Sequential Quadratic Programming}
\label{sec:PBSQP}
Building on the gradient-based philosophy of GIPBO, we build \prefsqp, which additionally exploits the posterior mean Hessian $\bm{\mu}_{\nabla^2 f}^{\mathrm{pref}}(\optvar_t)$ derived in Section~\ref{sec:pref_modeling}.
Concretely, it adapts BayeSQP~\cite{brunzema_bayesqp_2026} to the preference setting: scalar observations are replaced by pairwise comparisons, and a quadratic subproblem defined by the posterior mean gradient and Hessian of the pairwise GP determines the search direction.
To our knowledge, \prefsqp is the first method to incorporate second-order information from a preference surrogate into the optimization step.

\paragraph{Local Sub-Sampling via Pairwise Comparisons}
At each iteration, BayeSQP draws $M = d$ points uniformly from a $d$-dimensional ball of radius $\varepsilon$ centered at the current iterate $\optvar_t$.
In the standard (scalar-observation) setting, each sample produces a noisy function value.
\prefsqp instead compares every sampled point $\optvar_i$ to the current iterate, yielding a binary preference $\pi_i \in \{0,1\}$.
All $M$ comparisons $\{(\optvar_t, \optvar_i), \pi_i\}_{i=1}^{M}$ are then used to update the pairwise GP. %

\paragraph{Subproblem Formulation}
Following \citet{brunzema_bayesqp_2026}, we adopt the expected-value formulation of the BayeSQP subproblem.
Since our objective is to maximize the latent utility \(f\), \prefsqp seeks the local ascent direction
\begin{equation}
    \bm{p}_t \in \argmax_{\bm{p} \in \mathbb{R}^d}
    \quad
        \mathbb{E} \left[
            \frac{1}{2}\bm{p}^\top \nabla^2 f(\optvar_t)\bm{p}
            + \nabla f(\optvar_t)^\top \bm{p}
            + f(\optvar_t)
        \right]
    \label{eq:prefsqp_subproblem}
\end{equation}
We prefer this formulation over the uncertainty-aware variant of BayeSQP because the Laplace approximation underlying the pairwise GP already limits the fidelity of the posterior covariances, which can undermine robustness mechanisms such as the second-order cone program.
Moreover, the expected-value formulation performs strongly for unconstrained objectives in the scalar setting~\cite{brunzema_bayesqp_2026}.
Since $\bm{p}$ is deterministic, the expectation reduces to the posterior quantities
$\bm{H}_t \coloneqq \boldsymbol{\mu}_{\nabla^2 f}^{\mathrm{pref}}(\optvar_t)$
and
$\boldsymbol{\mu}_{\nabla f}^{\mathrm{pref}}(\optvar_t)$.
When $\bm{H}_t \prec 0$,\footnote{In practice, we follow \citet{brunzema_bayesqp_2026} and enforce this condition via an eigenvalue modification.}
the expected quadratic problem \eqref{eq:prefsqp_subproblem} has the closed-form maximizer
\begin{equation}
    \bm{p}_t
    =
    -\bm{H}_t^{-1}\boldsymbol{\mu}_{\nabla f}^{\mathrm{pref}}(\optvar_t)
    =
    -\left[
        \nabla_{\optvar_t}^2 K(\optvar_t,\bm{X})
        \bm{K}_{\bm{XX}}^{-1}
        \hat{\bm{f}}
    \right]^{-1}
    \nabla_{\optvar_t}K(\optvar_t,\bm{X})
    \bm{K}_{\bm{XX}}^{-1}
    \hat{\bm{f}},
    \label{eq:newton_direction}
\end{equation}
which is the classic Newton-type ascent direction obtained via a single linear solve.

\paragraph{Line Search via Pairwise Comparisons}
Given $\bm{p}_t$, a line search can be performed along the segment $\{\optvar_t + \eta\, \bm{p}_t \mid \eta \in [0, 1]\}$.
In Appendix~\ref{app:ablations}, we compare several strategies including Thompson sampling (TS) and EUBO~\cite{astudillo2023qEUBO} on the line as well as taking the full Newton step ($\eta = 1$).
All approaches yield comparable performance in our setting with preference feedback.
We therefore use TS as the line search method \prefsqp in all experiments. %

\paragraph{On Approximation Quality of the Hessian Prediction}
Whether second-order information from a Laplace-approximated pairwise GP is reliable enough to be useful in a downstream algorithm is not obvious.
Unlike GIPBO, where the \eqref{eq:GI_PBO_acf_detailed} actively reduces posterior gradient uncertainty at the current iterate, \prefsqp reads off the Hessian from a locally fitted pairwise GP without explicitly targeting curvature accuracy.
The following identity relates the error induced by replacing the exact posterior mean with the Laplace MAP estimate to the resulting gradient and Hessian predictions.
\begin{lemma}[Laplace approximation error]
\label{prop:approx_error}
Let $\bar{\bm{f}} := \mathbb{E}_{p(\bm{f}|\mathcal{D})}[\bm{f}]$ denote the exact posterior mean and $\hat{\bm{f}}$ the MAP estimate, and define the MAP bias $\bm{\delta} := \bar{\bm{f}} - \hat{\bm{f}}$.
Recall that
\(\mathcal{L}_{1,\optvar}f=\nabla f(\optvar)\),
\(\mathcal{L}_{2,\optvar}f=\operatorname{vec}\{\nabla^2 f(\optvar)\}\),
and
\(\bm{k}_{a}(\optvar,\bm{X})=\mathcal{L}_{a,\optvar}K(\optvar,\bm{X})\).
Then, for \(k\in\{1,2\}\),
\begin{equation}
    \boldsymbol{\mu}_{\mathcal{L}_k f}^{\mathrm{exact}}(\optvar)
    -
    \boldsymbol{\mu}_{\mathcal{L}_k f}^{\mathrm{pref}}(\optvar)
    =
    \bm{k}_{k}(\optvar,\bm{X})\,
    \bm{K}_{\bm{XX}}^{-1}\,
    \bm{\delta}.
    \label{eq:grad_hess_error_identity}
\end{equation}
\end{lemma}
\vspace{-1em}
\begin{proof}[Proof sketch]
By the law of iterated expectations, the exact predictive mean of $\mathcal{L}_{k,\optvar}f$ under the true (non-Gaussian) posterior is
$\boldsymbol{\mu}_{\mathcal{L}_k f}^{\mathrm{exact}}(\optvar)
=
\bm{k}_{k}(\optvar,\bm{X})\,\bm{K}_{\bm{XX}}^{-1}\,\bar{\bm{f}}$,
while the Laplace approximation replaces $\bar{\bm{f}}$ by $\hat{\bm{f}}$, giving
$\boldsymbol{\mu}_{\mathcal{L}_k f}^{\mathrm{pref}}(\optvar)
=
\bm{k}_{k}(\optvar,\bm{X})\,\bm{K}_{\bm{XX}}^{-1}\,\hat{\bm{f}}$.
The full proof is in Appendix~\ref{app:laplace_bound}.
\end{proof}
The identity~\eqref{eq:grad_hess_error_identity} shows that the gradient and Hessian errors are driven by the same MAP bias $\bm{\delta}$, but propagated through different operators $\bm{k}_{k}(\optvar,\bm{X})\,\bm{K}_{\bm{XX}}^{-1}$.
For the Hessian ($k=2$), this operator involves second derivatives of the kernel and can therefore be more sensitive to the kernel lengthscale, the conditioning of $\bm{K}_{\bm{XX}}$, and the local geometry of the sampled points than the gradient case ($k=1$).
Thus, the availability of a posterior mean Hessian does not by itself imply that the resulting search direction is reliable.
Whether the approximate curvature information is beneficial is therefore an empirical question; in Section~\ref{sec:results}, we find that it can be useful in practice.

\subsection{TuRPBO: Trust-Region Preferential Bayesian Optimization}
\label{sec:turpbo}
Lastly, we construct TuRPBO (pronounced ``turpe\underline{b}o''), as a preferential analogue of the popular TuRBO algorithm \cite{eriksson2019scalable}.
Unlike the gradient-based methods above, TuRPBO does not incorporate explicit higher-order information into the search.
Instead, it restricts the acquisition function to an adaptive trust region that expands after successful queries and contracts after failures, making it particularly straightforward to transfer to the preferential setting.
Within each trust region, we use either Thompson sampling \cite{eriksson2019scalable} or qEUBO \cite{astudillo2023qEUBO} as the acquisition function.

The key transfer from TuRBO to TuRPBO is the definition of a successful step.
In standard BO, a trust-region step is successful if a newly observed scalar value improves upon the best value observed so far.
In TuRPBO, scalar values are unavailable; instead, a candidate $\optvar_{t,i}$ is considered successful if it is preferred to the current best point, \ie if $\optvar_{t,i} \succ \optvar_t^*$.
For a batch $B_t=\{\optvar_{t,1},\ldots,\optvar_{t,q}\}$, the iteration is therefore successful if at least one candidate satisfies $\optvar_{t,i} \succ \optvar_t^*$; otherwise, $\optvar_t^*$ is left unchanged and the iteration is counted as a failure.
The trust region is then expanded after consecutive successes and contracted after consecutive failures, exactly as in TuRBO.
Full details are in Appendix~\ref{app:turpbo_full}.
This makes TuRPBO the most direct local-BO adaptation presented here: the trust-region geometry, lengthscale-based rescaling, and expansion--contraction rules remain unchanged, while the scalar improvement test is replaced by a preference-based one.

\section{Results}
\label{sec:results}

We benchmark our family of local PBO methods against standard global approaches to determine the settings in which local search is beneficial under preference-based feedback.
As baselines we include hallucination believer (HB-EI)~\cite{takeno2023towards}, qEUBO~\cite{astudillo2023qEUBO}, GLISp~\cite{bemporad2021global}, and pseudo-random Sobol sampling. 
We report online best-seen performance, $f_{\mathrm{best}, k}=\max_{k\leq K} f_k,$ as well as cumulative performance $F_k=\sum_{k=0}^{K} f_k$ over function evaluations.
The latter captures the quality of the full optimization trajectory and is therefore important in controller tuning and policy search problems, where poor intermediate evaluations can be costly.
Note that the algorithms only get comparisons of noisy function values \eqref{eq:noise}. We repeat each experiment 10 times with a new initialization.
The added noise is Gaussian with a standard deviation \SI{10}{\percent} of the approximated function value range. All methods are evaluated with matched budgets and the same initial data. %
Full experimental details are provided in Appendix~\ref{app:experiments}.
We further include ablations in Appendix~\ref{app:ablations}.

\paragraph{Within-Model and Out-of-Model Comparisons}
\label{sec:within}

We first study the controlled setting of GP sample-path problems in dimensions $8$ to $96$ with squared-exponential kernels ($\ell=0.1$ and $\ell=0.5$), under both within-model (known hyperparameters) and out-of-model (learned hyperparameters) settings.
Figure~\ref{fig:gp_comp} reports best-seen performance after $2+10d$ evaluations.
Local methods excel in high-dimensional, short-lengthscale problems, where global acquisition functions become inefficient; PrefSQP performs particularly well here, while TuRPBO is the most robust overall.
For very smooth functions, locality is less beneficial, and with only two initial points, local methods can get trapped in poor local optima and underperform global baselines. Contrary to the derivative-based local PBO approaches, the wide initial trust region of TuRPBO allows it to escape a bad initialization and then commit to a good local region, highlighting that combined local approaches are interesting future work (\cf Sec.~\ref{sec:limitations}).
The out-of-model comparison further reveals that qEUBO and GIPBO are sensitive to hyperparameter learning. %
HB-EI and GIPBO are evaluated only up to $32$ dimensions due to computational cost.
See Appendix~\ref{app:gp_experiments} for full optimization curves.

\paragraph{Evaluation on Standard Benchmarks} %
\label{sec:standard_benchmarks}

We next evaluate on standard benchmarks starting from two random initial points and optimizing for $10d$ iterations.
Figure~\ref{fig:synthetic} reports best-seen and cumulative performance, confirming the main trend from the GP study: local methods become more useful as dimensionality and landscape complexity increase, achieving better cumulative performance by exploiting promising regions early rather than globally exploring the full search space.
PrefSQP consistently outperforms GIPBO, indicating that second-order posterior information is more effective for local refinement than estimating gradients alone despite the Laplace-approximate GP posterior, and is mostly competitive with TuRPBO, outperforming it in the highest-dimensional setting.
The main exception is Hartmann, where TuRPBO performs best, consistent with the results on smoother objectives from Sec.~\ref{sec:within}.
Table~\ref{tab:compute-time} reports wall-clock time: Sobol and GLISp are inexpensive, while HB-EI, and GIPBO become costly as dimension increases.
qEUBO, PrefSQP, and TuRPBO remain feasible also in high dimensions.
It also highlights a trade-off: derivative-informed PBO can improve sample efficiency, but not all derivative-informed methods scale equally well computationally.
\begin{figure}[t]
    \centering
    \includegraphics[width=\linewidth, trim=0 0.1cm 0 0.25cm, clip]{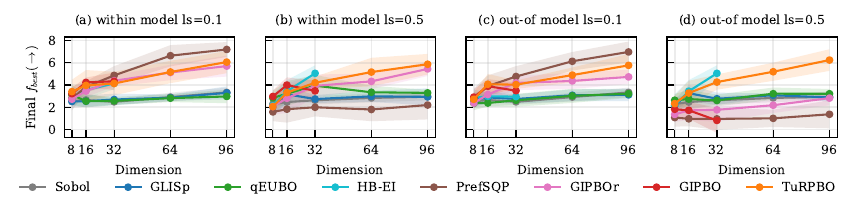}
    \vspace{-1em}
    \caption{\textit{Performance on GP sample paths.} We plot the final performance (mean $\pm$ stdv) after $10d$ iterations.}
    \vspace{-1em}
    \label{fig:gp_comp}
\end{figure}

\begin{table}[t]
\centering
\caption{Compute time in seconds per algorithm across benchmark functions (mean $\pm$ stdv).}
\resizebox{\textwidth}{!}{%
\begin{tabular}{lcccccccc}
\toprule
& & \multicolumn{3}{c}{Global methods} & \multicolumn{4}{c}{Local methods} \\
\cmidrule(lr){3-5}\cmidrule(lr){6-9}
Function & Sobol & GLISp & qEUBO ($q=1$) & HB-EI & GIPBO & GIPBOr & TuRPBO & PrefSQP \\
\midrule
Hartmann 6D & 0.023 $\pm$ 0.003 & 9.0 $\pm$ 5.2 & 33 $\pm$ 5 & 175 $\pm$ 17 & 245 $\pm$ 23 & 2.6 $\pm$ 0.8 & 46 $\pm$ 8 & 22 $\pm$ 19 \\
Rosenbrock 16D & 0.059 $\pm$ 0.004 & 7.5 $\pm$ 0.5 & 171 $\pm$ 83 & 716 $\pm$ 40 & 2453 $\pm$ 305 & 4.8 $\pm$ 2.3 & 143 $\pm$ 9 & 66 $\pm$ 10 \\
Levy 32D & 0.14 $\pm$ 0.01 & 20 $\pm$ 2 & 707 $\pm$ 302 & 7550 $\pm$ 685 & 15556 $\pm$ 2408 & 12 $\pm$ 8 & 520 $\pm$ 117 & 553 $\pm$ 394 \\
StyblinskiTang 64D & 0.29 $\pm$ 0.03 & 94 $\pm$ 7 & 4921 $\pm$ 2567 & -- & -- & 34 $\pm$ 8 & 3162 $\pm$ 592 & 2918 $\pm$ 538 \\
\bottomrule
\end{tabular}
}
\label{tab:compute-time}
\vspace{-1em}
\end{table}

\paragraph{Policy Search}
\label{sec:results_rl}
\begin{figure}
    \centering
    \includegraphics[width=\textwidth, trim=0.1cm 0.2cm 0 0.26cm, clip]{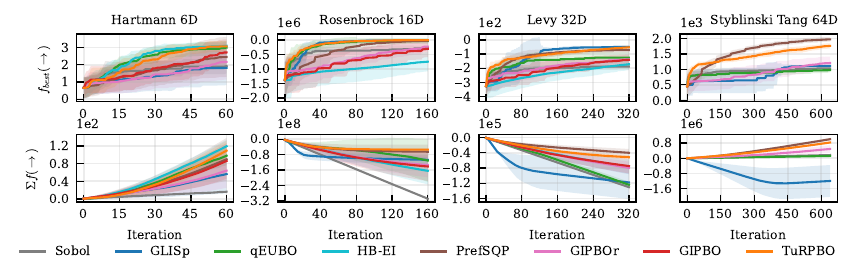}%
    \vspace{-1em}
    \caption{\textit{Performance on synthetic benchmarks.} Best (top) and cumulative (bottom) performance (mean $\pm$ stdv).
    Our local PBO methods achieve higher cumulative performance in high-dimensional settings.}
    \label{fig:synthetic}
    \vspace{-1em}
\end{figure}

\begin{figure}
    \centering
    \includegraphics[width=\textwidth, trim=0 0.2cm 0 0.26cm, clip]{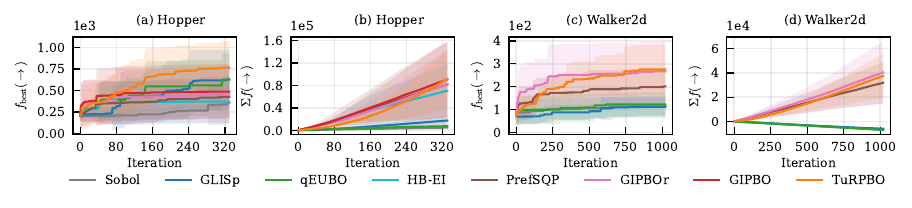}
    \vspace{-2em}
    \caption{\textit{Performance on policy-search tasks Hopper and Walker2D.}}
    \label{fig:rl_bench}
    \vspace{-2em}
\end{figure}

Lastly, we evaluate the baselines on MuJoCo Hopper ($33\text{-}d$) and Walker2D ({$102\text{-}d$}) policy-search tasks, using linear feedback policies following~\cite{horia2018simple}, where each iteration is one full episode rollout.
Each run is warm-started with $5d$ random evaluations; the best observed policy initializes the preference model, and each method runs for $10d$ iterations, simulating a real-world setting with a crude initial policy.
Figure~\ref{fig:rl_bench} shows that TuRPBO achieves the strongest final performance in both tasks.
Overall, our local PBO methods obtain higher cumulative performance than global baselines by avoiding low-return evaluations.
GLISp is competitive in final performance but has substantially worse cumulative performance, which can be problematic in practice.
Non-warm-started results in Appendix~\ref{app:policy_details} show the complementary failure mode: without a promising initial region, purely local methods are less reliable.
Thus, local PBO is most useful for high-dimensional policy tuning when prior knowledge or an existing controller provides a reasonable starting region.

\section{Discussion on Limitations and Future Work}
\label{sec:limitations}

While our results show that local search mechanisms can improve PBO in high-dimensional settings, some limitations remain. First, we studied the proposed local mechanisms separately to isolate their individual strengths and weaknesses. In practice, however, combinations of global search or warm-starting with a subsequent local gradient-based refinement \cite{DELCARO2026106891} may lead to further performance improvements as we saw that given a good initial guess, these local PBO approaches were highly effective. Designing such hybrid methods is an important direction for future work. Second, our evaluations build on synthetic objectives and policy-search problems.
Although these benchmarks are useful for controlled comparisons, they do not fully capture the structure of human preferences. In real applications \cite{coutinho2024human, erarslan_consecutive_2025,dewancker2018sequential}, experiments and preferences may be noisier, inconsistent, context-dependent, or limited by a threshold of discernibility, where two alternatives are effectively indistinguishable to the user.
Further, some real-world problems are prone to experimental ``crashes'' \cite{menn2026preferential}, motivating a more cautious exploration through local search.
Future work should evaluate local PBO on real systems with human feedback on high-dimensional problems. Finally, the effect of alternative comparison mechanisms, such as larger comparison batches \cite{astudillo2023qEUBO}, ranking \cite{nguyen2021top}, or explicit indifference responses \cite{bemporad2021global, erarslan_consecutive_2025}
on user burden and robustness
remains an interesting direction for future work.

\section{Conclusion}
\label{sec:conclusion}

This work presented a family of local PBO methods that adapt trust-region optimization and derivative-informed refinement to pairwise preference feedback.
Our experiments show that restricting search to promising regions addresses a central scalability challenge of PBO, yielding clear gains in high-dimensional, complex landscapes.
Among the proposed methods, TuRPBO offers the most robust performance, while PrefSQP demonstrates that second-order posterior information from a Laplace-approximated pairwise GP can meaningfully improve local refinement.
These findings position local PBO as well suited for applications such as controller tuning and policy search, where reasonable initial solutions often exist and costly intermediate evaluations must be avoided.
However, locality is not universally beneficial: in low dimensions or without a good initial region, global methods remain competitive, pointing to hybrid global-local strategies as a natural next step.

\section*{Acknowledgments}
This work was funded in part by the Deutsche Forschungsgemeinschaft (DFG, German Research Foundation) under Germany's Excellence Strategy -- EXC-2023 Internet of Production -- 390621612, project 556142469 (ParDyBO) and RTG 2236/2 (UnRAVeL). Computations were performed using resources granted by RWTH Aachen University under Project rwth2055 and p0021919. This work was conducted in part within the Helmholtz School for Data Science in Life, Earth and Energy (HDS-LEE).)
\bibliography{neurips}

\clearpage
\appendix

\section{Broader Societal Impacts of the Developed Local PBO Algorithms}\label{app:impacts}
Our paper proposes three local PBO algorithms for solving optimization problems using only pairwise comparisons.
These methods could make optimization more practical in high-dimensional settings where numerical objectives are difficult to define but comparisons between alternatives are feasible, particularly when human judgment is involved.
By learning from such comparisons, local PBO may support the optimization of qualitative criteria such as comfort, usability, safety, or task performance, while its focus on local refinement can help reduce poor intermediate trials in applications such as controller tuning and policy search.
Potential negative impacts mainly arise from misuse or biased deployment of preference-based optimization, for example when learned preferences reflect unfair, unsafe, or manipulative objectives, or when user preferences are exploited in harmful ways.
Overall, the work is methodological and does not target a specific high-risk application; its societal impact will depend largely on how it is deployed.

\section{Proof of Lemma~\ref{prop:approx_error}}
\label{app:laplace_bound}

We restate the lemma and provide the full proof.
Recall the notation from Section~\ref{sec:pref_modeling}:
\(\mathcal{L}_{1,\optvar}f=\nabla f(\optvar)\),
\(\mathcal{L}_{2,\optvar}f=\operatorname{vec}\{\nabla^2 f(\optvar)\}\),
and
\(\bm{k}_{a}(\optvar,\bm{X})=\mathcal{L}_{a,\optvar}K(\optvar,\bm{X})\).

\begin{assumption}
\label{asmp:kernel_diff}
The covariance function $k(\optvar, \optvar')$ is at least four times differentiable in its first argument.
\end{assumption}

\begin{assumption}
\label{asmp:K_invertible}
The prior covariance matrix $\bm{K}_{\bm{XX}} = K(\bm{X},\bm{X})$ is strictly positive definite.
\end{assumption}

\begin{assumption}
\label{asmp:posterior_integrability}
The exact posterior $p(\bm{f} \mid \mathcal{D})$ satisfies $\mathbb{E}_{p(\bm{f} \mid \mathcal{D})}[\|\bm{f}\|_2] < \infty$, so that the posterior mean $\bar{\bm{f}} := \mathbb{E}_{p(\bm{f} \mid \mathcal{D})}[\bm{f}]$ is well defined and the interchange of differentiation and integration below is justified.
\end{assumption}

\begin{lemma*}[Laplace approximation error for gradient and Hessian predictions, restated]
Under Assumptions~\ref{asmp:kernel_diff}--\ref{asmp:posterior_integrability}, let $\bar{\bm{f}} := \mathbb{E}_{p(\bm{f}|\mathcal{D})}[\bm{f}]$ denote the exact posterior mean and $\hat{\bm{f}}$ the MAP estimate, and define the MAP bias $\bm{\delta} := \bar{\bm{f}} - \hat{\bm{f}}$.
Then, for any test point $\optvar \in \Theta$ and $k \in \{1,2\}$,
\begin{equation}
    \boldsymbol{\mu}_{\mathcal{L}_k f}^{\mathrm{exact}}(\optvar)
    -
    \boldsymbol{\mu}_{\mathcal{L}_k f}^{\mathrm{pref}}(\optvar)
    =
    \bm{k}_{k}(\optvar, \bm{X})\, \bm{K}_{\bm{XX}}^{-1}\, \bm{\delta}.
    \tag{\ref{eq:grad_hess_error_identity}}
\end{equation}
\end{lemma*}

\begin{proof}
The proof proceeds in two steps: we first establish the predictive mean under an arbitrary posterior and then pass to derivatives to obtain the exact identity~\eqref{eq:grad_hess_error_identity}. %

For a GP with prior mean zero and prior covariance $k$, the conditional distribution of the function value $f(\optvar)$ at a test point $\optvar$, given the latent vector $\bm{f}$ at the training locations $\bm{X}$, is
\begin{equation}
    p(f(\optvar) \mid \bm{f})
    =
    \mathcal{N}\!\bigl(
        K(\optvar, \bm{X})\, \bm{K}_{\bm{XX}}^{-1}\, \bm{f},\;
        k(\optvar, \optvar) - K(\optvar, \bm{X})\, \bm{K}_{\bm{XX}}^{-1}\, K(\bm{X}, \optvar)
    \bigr).
    \label{eq:conditional_f}
\end{equation}
This conditional is a property of the GP prior and holds regardless of the form of the posterior $p(\bm{f} \mid \mathcal{D})$.
Taking the predictive expectation under \emph{any} posterior $p(\bm{f} \mid \mathcal{D})$:
\begin{align}
    \mathbb{E}[f(\optvar) \mid \mathcal{D}]
    &=
    \int f(\optvar)\, p(f(\optvar) \mid \bm{f})\, p(\bm{f} \mid \mathcal{D})\, d\bm{f}
    \notag \\
    &=
    \int K(\optvar, \bm{X})\, \bm{K}_{\bm{XX}}^{-1}\, \bm{f} \;\, p(\bm{f} \mid \mathcal{D})\, d\bm{f}
    \notag \\
    &=
    K(\optvar, \bm{X})\, \bm{K}_{\bm{XX}}^{-1}\, \mathbb{E}_{p(\bm{f} \mid \mathcal{D})}[\bm{f}].
    \label{eq:pred_mean_any_posterior}
\end{align}
In particular, the exact predictive mean is
\begin{equation}
    \mu_f^{\mathrm{exact}}(\optvar)
    =
    K(\optvar, \bm{X})\, \bm{K}_{\bm{XX}}^{-1}\, \bar{\bm{f}},
    \label{eq:exact_pred_mean}
\end{equation}
and the Laplace-approximate predictive mean is
\begin{equation}
    \mu_f^{\mathrm{pref}}(\optvar)
    =
    K(\optvar, \bm{X})\, \bm{K}_{\bm{XX}}^{-1}\, \hat{\bm{f}}.
    \label{eq:lap_pred_mean}
\end{equation}

Under Assumptions~\ref{asmp:kernel_diff} and~\ref{asmp:posterior_integrability}, the linear differential operators $\mathcal{L}_{k,\optvar}$ can be interchanged with the posterior expectation.
For $k = 1$ (gradient, $\mathcal{L}_{1,\optvar}f = \nabla f(\optvar)$):
\begin{align}
    \boldsymbol{\mu}_{\mathcal{L}_1 f}^{\mathrm{exact}}(\optvar)
    &=
    \mathcal{L}_{1,\optvar}\, \mu_f^{\mathrm{exact}}(\optvar)
    =
    \bm{k}_{1}(\optvar, \bm{X})\, \bm{K}_{\bm{XX}}^{-1}\, \bar{\bm{f}},
    \label{eq:exact_grad_mean}
    \\
    \boldsymbol{\mu}_{\mathcal{L}_1 f}^{\mathrm{pref}}(\optvar)
    &=
    \mathcal{L}_{1,\optvar}\, \mu_f^{\mathrm{pref}}(\optvar)
    =
    \bm{k}_{1}(\optvar, \bm{X})\, \bm{K}_{\bm{XX}}^{-1}\, \hat{\bm{f}}.
    \label{eq:lap_grad_mean}
\end{align}
For $k = 2$ (vectorized Hessian, $\mathcal{L}_{2,\optvar}f = \operatorname{vec}\{\nabla^2 f(\optvar)\}$):
\begin{align}
    \boldsymbol{\mu}_{\mathcal{L}_2 f}^{\mathrm{exact}}(\optvar)
    &=
    \bm{k}_{2}(\optvar, \bm{X})\, \bm{K}_{\bm{XX}}^{-1}\, \bar{\bm{f}},
    \label{eq:exact_hess_mean}
    \\
    \boldsymbol{\mu}_{\mathcal{L}_2 f}^{\mathrm{pref}}(\optvar)
    &=
    \bm{k}_{2}(\optvar, \bm{X})\, \bm{K}_{\bm{XX}}^{-1}\, \hat{\bm{f}}.
    \label{eq:lap_hess_mean}
\end{align}
Subtracting \eqref{eq:lap_grad_mean} from \eqref{eq:exact_grad_mean} (respectively \eqref{eq:lap_hess_mean} from \eqref{eq:exact_hess_mean}) yields, for $k \in \{1,2\}$,
\begin{equation}
    \boldsymbol{\mu}_{\mathcal{L}_k f}^{\mathrm{exact}}(\optvar)
    -
    \boldsymbol{\mu}_{\mathcal{L}_k f}^{\mathrm{pref}}(\optvar)
    =
    \bm{k}_{k}(\optvar, \bm{X})\, \bm{K}_{\bm{XX}}^{-1}\,
    (\bar{\bm{f}} - \hat{\bm{f}})
    =
    \bm{k}_{k}(\optvar, \bm{X})\, \bm{K}_{\bm{XX}}^{-1}\, \bm{\delta},
\end{equation}
which is the exact identity~\eqref{eq:grad_hess_error_identity}.

\end{proof}

\newpage
\section{Implementation Details and Full Algorithms}
\label{app:implementation}

\subsection{Overview of Implementation Choices}
All experiments are performed on normalized domains $\Theta=[0,1]^d$, and all algorithms receive only noisy pairwise comparisons rather than direct function values. The local methods use the Laplace-approximated pairwise GP surrogate introduced in Section~\ref{sec:pref_modeling} to model the latent utility from these comparisons. The following sections provide implementation details that are omitted from the main text, including acquisition optimization, local sampling strategies, step normalization, trust-region updates, and the full algorithmic procedures used in the experiments.

\subsection{GIPBO: Implementation Details}
\label{sec:GIPBO_implementation_details}

\subsubsection{Algorithmic overview}
At a high level, one GIPBO iteration consists of a local gradient-estimation phase followed by a gradient step to update the center point. Figure \ref{fig:overview_gibo} illustrates the algorithm and the full algorithm is described in Algorithm \ref{alg:gipbo}. The algorithm starts with one comparison (Figure \ref{fig:1a}). The winner of this comparison is then selected as the current center point $x_t$. Starting from this point, GIPBO sequentially selects $M=d$ local query points using the preferential GI acquisition function (Figure \ref{fig:1b}). Each query is compared against the current local incumbent, the pairwise GP is refit, and the next query is selected under the updated posterior. After the $d$ comparisons, the posterior mean gradient at $x_t$ defines a normalized ascent direction, and the next center point is obtained by a projected fixed-length step (Figure \ref{fig:1c}). The step itself is not queried immediately; it becomes the center of the next gradient-estimation phase. 
\begin{figure*}[t]
  \centering
\begin{subfigure}[t]{0.24\textwidth} \plotwithaxislabels{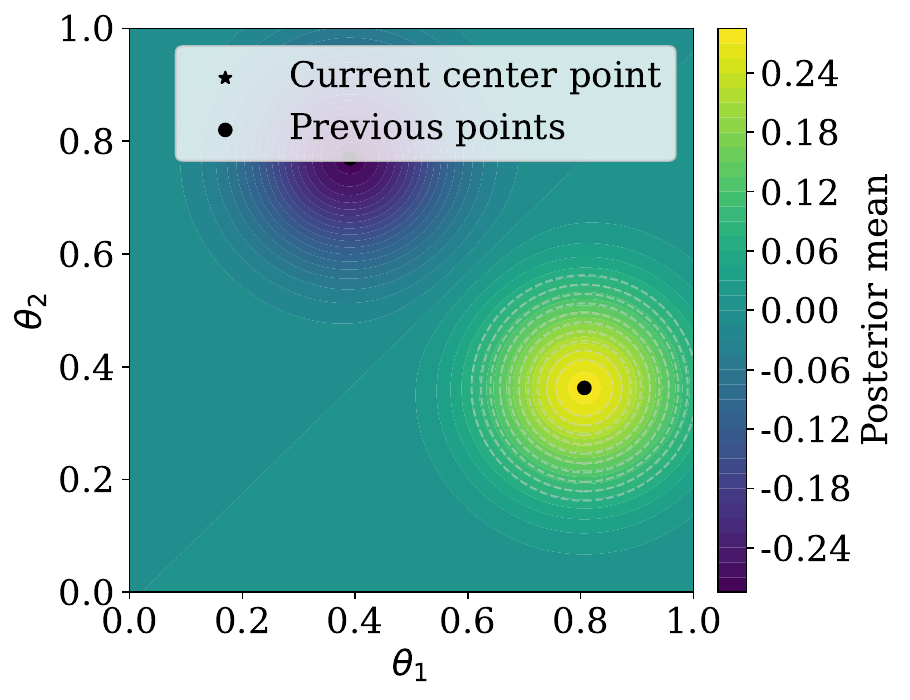}
    \caption{Initial GP posterior}
    \label{fig:1a}
  \end{subfigure}
  \hfill
  \begin{subfigure}[t]{0.24\textwidth}
    \plotwithaxislabels{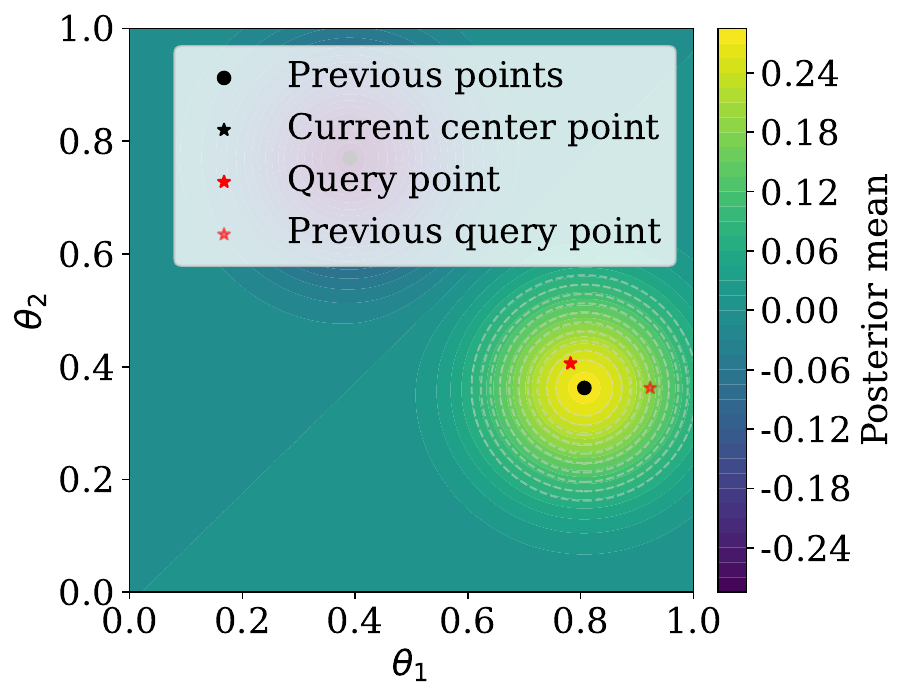}
    \caption{Gradient estimation}
    \label{fig:1b}
  \end{subfigure}
  \hfill
  \begin{subfigure}[t]{0.24\textwidth}
    \plotwithaxislabels{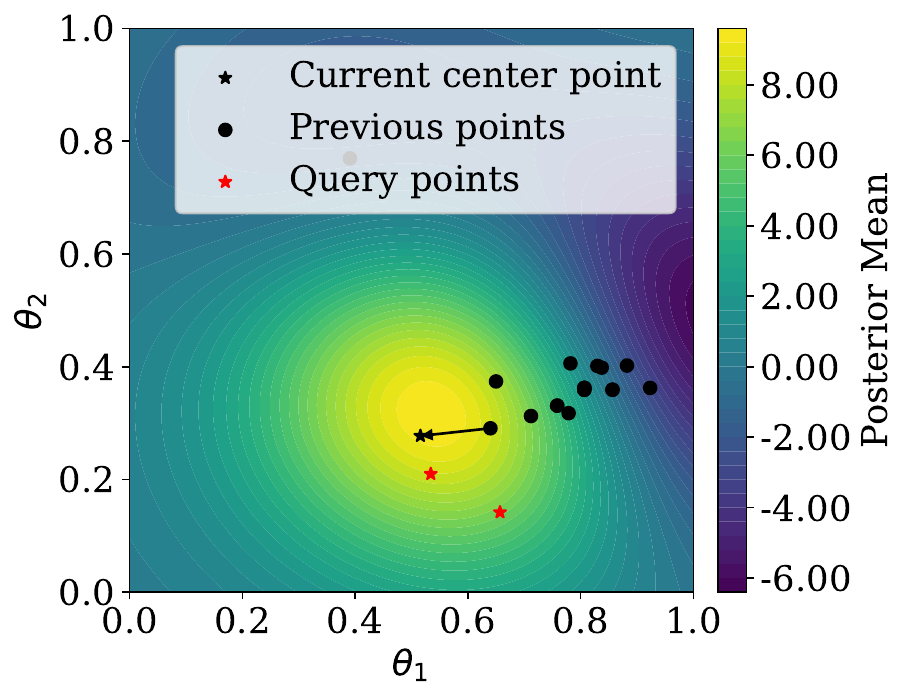}
    \caption{Gradient step}
    \label{fig:1c}
  \end{subfigure}
  \hfill
  \begin{subfigure}[t]{0.24\textwidth}
    \plotwithaxislabels{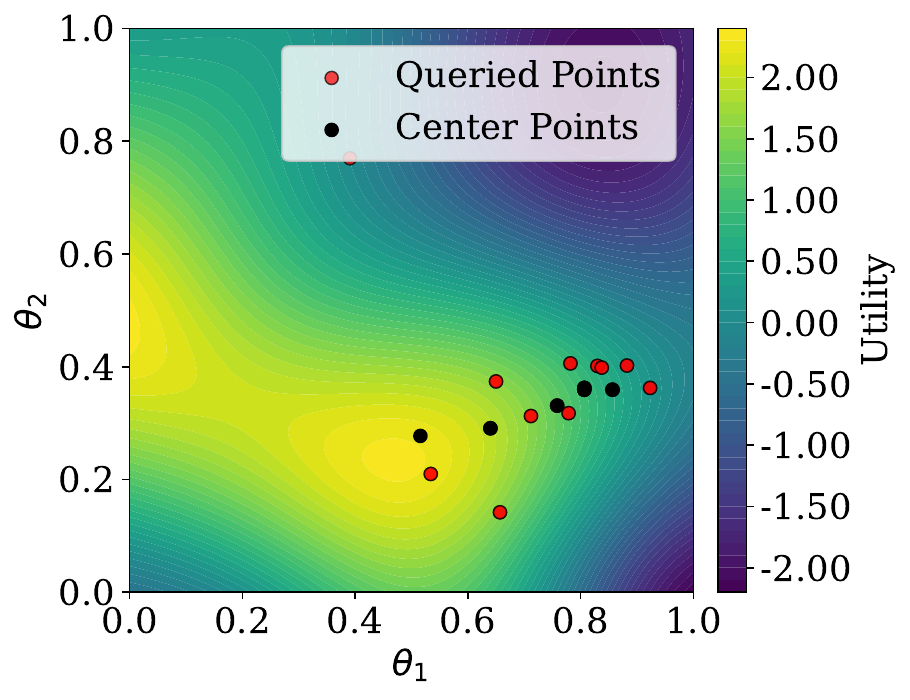}
    \caption{Result on true function}
    \label{fig:1d}
  \end{subfigure}
  
  \caption{Illustration of gradient steps in GIPBO. (a) The GP posterior is initialized with two data points. (b) A gradient estimate is obtained by querying $d$ samples (red query points) computed by the adapted GI acquisition function. (c) A gradient step is performed. (d) The result of five gradient steps is visualized on the true function's plot.
  As we do not observe function measurements but only comparisons, the output scale of the GP gets rescaled in every iteration.}
  \label{fig:overview_gibo}
\end{figure*}

\subsubsection{Preferential GI acquisition evaluation}
In contrast to GIBO, the matrix
\((\bm{K}_{\hat{\bm{X}}\hat{\bm{X}}}+\hat{\cthingy}^{-1})^{-1}\) in the preferential Jacobian covariance
has to be recomputed separately for both possible virtual preference outcomes at each candidate
\(x\). Since \(\hat C\) depends on the outcome of the duel, the efficient Cholesky updates used in
GIBO are not directly applicable. For each virtual dataset, we therefore compute
\((\bm{K}_{\hat{\bm{X}}\hat{\bm{X}}}+\hat{\cthingy}^{-1})^{-1}\) by a full Cholesky decomposition when evaluating
\(\bm{\Sigma}^{\mathrm{pref}}_{11}\).

\subsubsection{Candidate-set construction}
For the acquisition function optimization, we sample virtual candidates from multiple \(d\)-dimensional
hyperspheres centered at \(\optvar_t\). To respect discernibility limits, we enforce a minimum radius of
\(0.05\). In two dimensions, we draw 2,500 candidates from 10 concentric circles with radii spanning
the local search bounds. Figure~\ref{fig:gp_mean_acq_comparison} illustrates this sampling procedure.
For \(d>2\), we sample \(d \times r\) candidates, where the resolution is fixed to \(r=50\). To reduce
the cost compared with uniform sampling over the full local search region, we restrict the candidates
to four hyperspheres with radii \(0.05, 0.1, 0.15,\) and \(0.2\). Points on each hypersphere are generated
using the Marsaglia method \cite{marsaglia1972sphere}.
In high-dimensional experiments, we cap the number of evaluated acquisition candidates at 512 for computational reasons. This cap makes full GIPBO less reliable in high dimensions because the capped candidate set can become too sparse to approximate the acquisition maximizer accurately.
 \begin{figure}[h]
  \centering
  \setlength{\tabcolsep}{10pt}
    \begin{tabular}{cc}
    \parbox{0.39\linewidth}{\centering
      \plotwithaxislabelstwo{\linewidth}{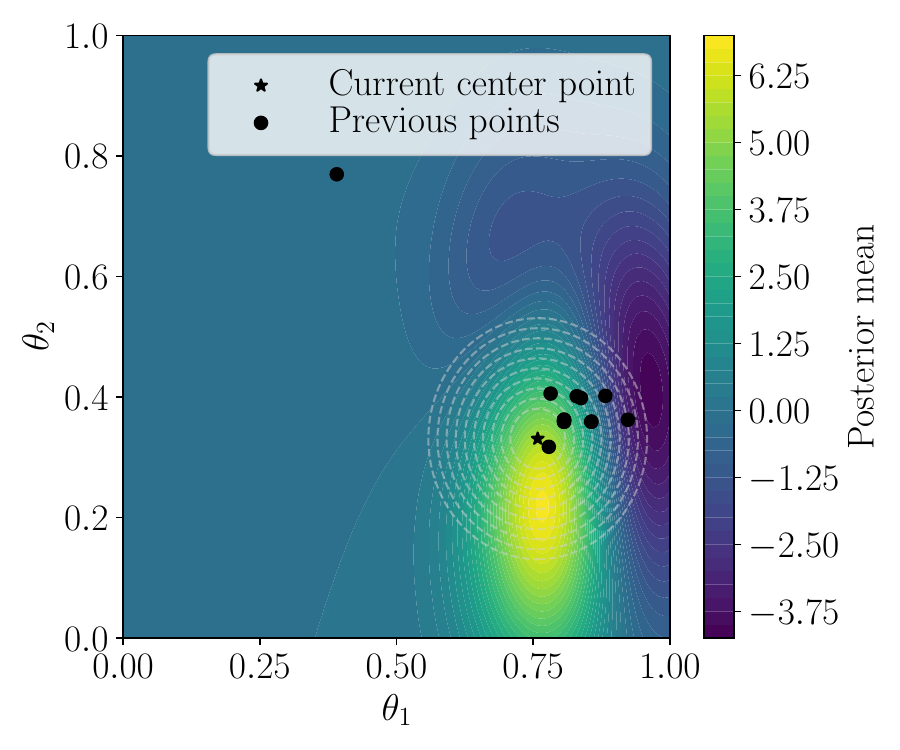}{0.45}{0.05}{0.045}{0.54}
    } &
    \parbox{0.39\linewidth}{\centering
      \plotwithaxislabelstwo{\linewidth}{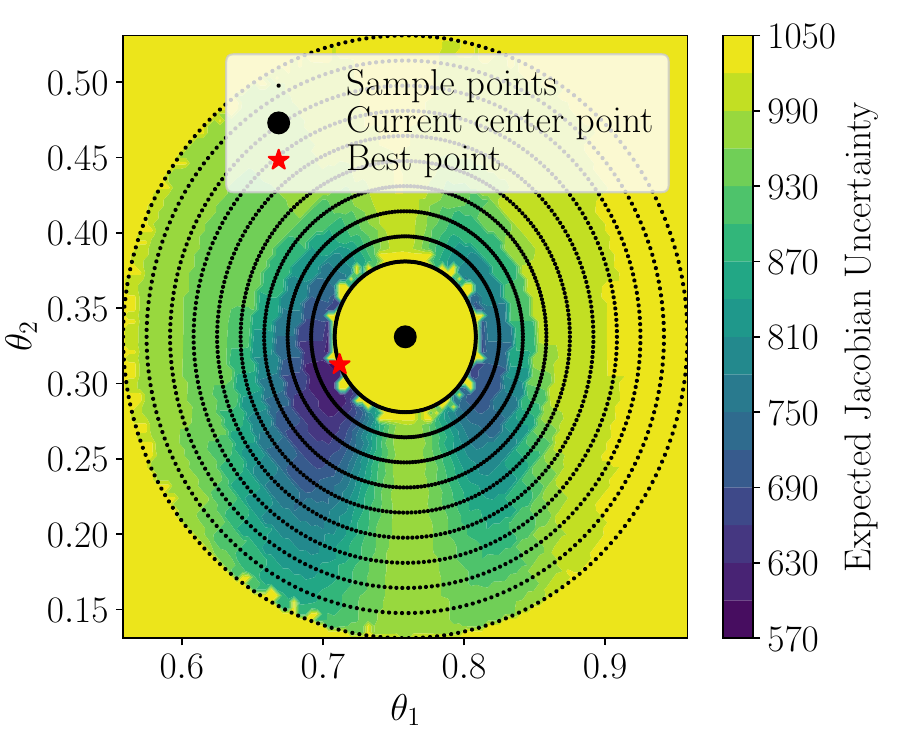}{0.45}{0.05}{0.03}{0.54}
    } \\

    \parbox{0.39\linewidth}{
      (a) GP posterior mean after the first three gradient steps. The search area for the acquisition function is displayed in light grey.
    } &
    \parbox{0.39\linewidth}{
      (b) Computation of the first query point to minimize the expected Jacobian uncertainty at the current parameters, here referred to as center point.
    } \\

    \addlinespace[4pt]

    \parbox{0.39\linewidth}{\centering
      \plotwithaxislabelstwo{\linewidth}{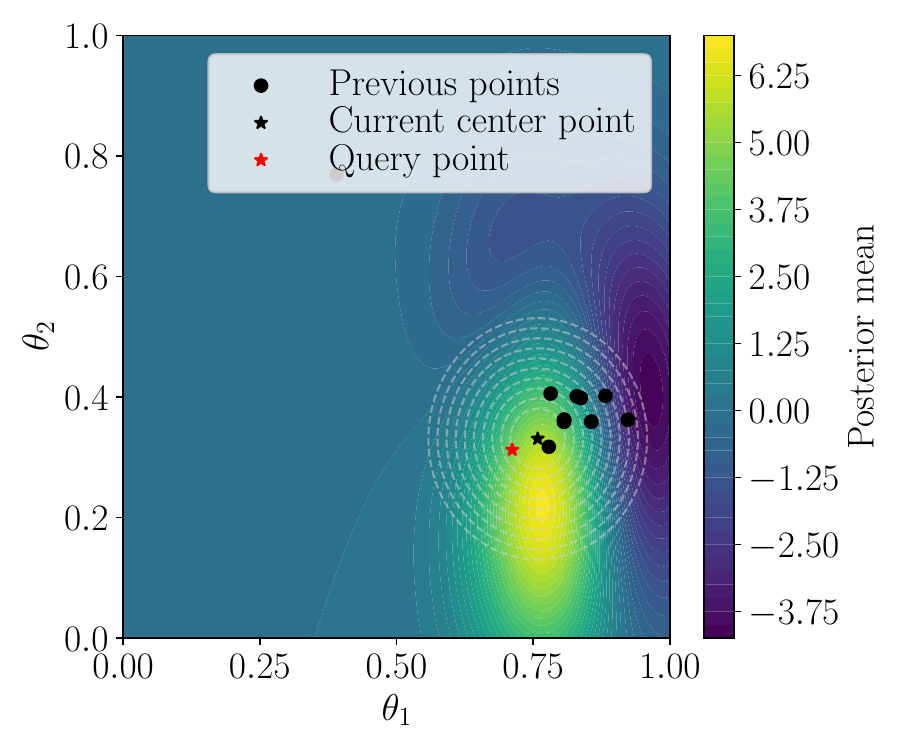}{0.45}{0.05}{0.045}{0.54}
    } &
    \parbox{0.39\linewidth}{\centering
      \plotwithaxislabelstwo{\linewidth}{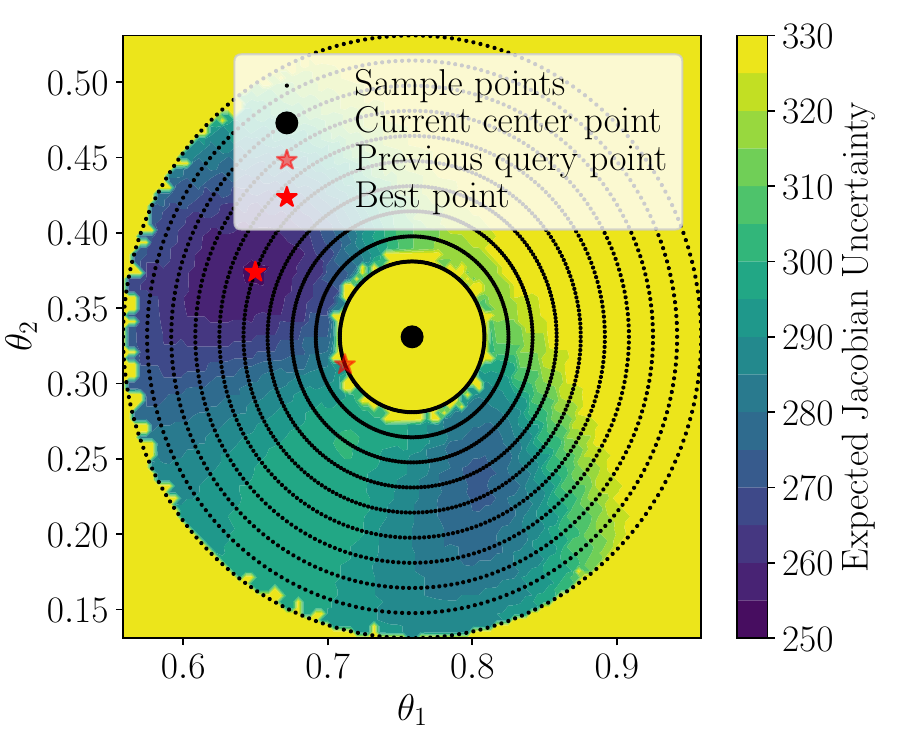}{0.45}{0.05}{0.03}{0.54}
    } \\

    \parbox{0.39\linewidth}{
      (c) Result of GP update with the duel of the center and first query point.
    } &
    \parbox{0.39\linewidth}{
      (d) Computation of the second query point needed for the gradient estimation.
    } \\
  \end{tabular}
  \caption[GP posterior mean and corresponding computation of the two query points to perform the fourth gradient step.]
  {GP posterior mean for parameters $\optvar = \{x_0, x_1\}$ and corresponding computation of the two query points to perform the fourth gradient step.}
  \label{fig:gp_mean_acq_comparison}
\end{figure}

\subsubsection{Gradient-step normalization}
We use a squared-exponential kernel on the normalized domain \(\Theta=[0,1]^d\), with
lengthscales constrained to \([0.05,0.5]\). This constraint avoids excessively large local steps,
since the gradient direction is normalized with respect to the squared-exponential lengthscales,
following \citet{muller2021local}. We set the minimum absolute step length to \(0.05\) and rescale
steps whenever they would fall below this threshold. Given the posterior mean gradient at the
current iterate \(\optvar_t\),
\[
\bm{g}_t = \bm{\mu}^{\mathrm{pref}}_{\nabla f}(\optvar_t),
\]
we use the Mahalanobis norm induced by the squared-exponential lengthscale matrix \(L\),
\[
\bm{p}_t = \frac{\bm{g}_t}{\|\bm{g}_t\|_L},
\qquad
\|\bm{g}_t\|_L = \sqrt{\bm{g}_t^\top L \bm{g}_t},
\]
and update the iterate according to \(\optvar_{t+1}=\Pi_{\Theta}(\optvar_t+\eta \bm{p}_t)\).

\subsection{GIPBOr: Randomized GIPBO Variant}
\label{app:gipbor_impl}

GIPBOr uses the same pairwise GP model, local incumbent update, and normalized gradient step as GIPBO, but replaces the expensive maximization of the preferential GI acquisition function by random local candidate selection. In each gradient-estimation phase, query points are sampled from the same local search region used by GIPBO. This variant isolates the effect of the gradient-based local update from the cost of acquisition optimization and provides a scalable approximation in higher-dimensional settings, where evaluating both virtual preference outcomes for many candidates becomes prohibitively expensive.

\subsubsection{Full algorithm}
\begin{algorithm}[H]
\caption{GIPBO}
\label{alg:gipbo}
\begin{algorithmic}[1]
\STATE Initialize data set \(\mathcal D_0\), current iterate \(\optvar_0\), step length \(\eta\), acquisition search bounds \((\delta_{\min},\delta_{\max})\)
\FOR{\(t=0,\ldots,T-1\)}
    \STATE Set local incumbent \(\optvar^\star_{t,0}\leftarrow \optvar_t\)
    \FOR{\(i=1,\ldots,M\)}
        \STATE Fit pairwise GP model using \(\mathcal D_t\)
        \STATE Select local query
        \[
            \optvar_{t,i}
            \in
            \arg\max_{\optvar \in \mathcal B(\optvar_t;\delta_{\min},\delta_{\max})\cap\Theta}
            \alpha_{\mathrm{GI\text{-}PBO}}(\optvar \mid \optvar_t,\mathcal D_t)
        \]
        \STATE Compare local query \(\optvar_{t,i}\) against local incumbent \(\optvar^\star_{t,i-1}\)
        \STATE Observe ordered comparison \((\optvar^+_{t,i},\optvar^-_{t,i})\)
        \STATE Update \(\mathcal D_t \leftarrow \mathcal D_t \cup \{(\optvar^+_{t,i},\optvar^-_{t,i})\}\)
        \IF{\(\optvar_{t,i}\) is preferred to \(\optvar^\star_{t,i-1}\)}
            \STATE \(\optvar^\star_{t,i}\leftarrow \optvar_{t,i}\)
        \ELSE
            \STATE \(\optvar^\star_{t,i}\leftarrow \optvar^\star_{t,i-1}\)
        \ENDIF
    \ENDFOR
    \STATE Fit pairwise GP model using \(\mathcal D_t\)
    \STATE Compute posterior mean gradient \(\bm{g}_t \leftarrow \bm{\mu}^{\mathrm{pref}}_{\nabla f}(\optvar_t)\)
    \STATE Normalize gradient direction \(\bm{p}_t \leftarrow \bm{g}_t / \| \bm{g}_t\|_L\)
    \STATE Update iterate \(\optvar_{t+1}\leftarrow \Pi_\Theta(\optvar_t+\eta \bm{p}_t)\)
    \STATE Set \(\mathcal D_{t+1}\leftarrow\mathcal D_t\)
\ENDFOR
\STATE \textbf{return} best observed incumbent \(\optvar^\star\)
\end{algorithmic}
\end{algorithm}

\subsection{PrefSQP: Implementation Details}
\label{app:prefsqp_impl}

PrefSQP follows the same preference-modeling setup as GIPBO but uses both the posterior mean gradient and posterior mean Hessian of the Laplace-approximated pairwise GP. At each iteration, $d$ local points are sampled in a ball of radius $\epsilon$ around the current iterate and compared to the current iterate. The resulting comparisons are added to the dataset and the pairwise GP is refit.

Given the posterior mean gradient $\bm{g}_t=\bm{\mu}^{\mathrm{pref}}_{\nabla f}(\optvar_t)$ and Hessian $\bm{H}_t=\bm{\mu}^{\mathrm{pref}}_{\nabla^2 f}(\optvar_t)$, PrefSQP constructs a local quadratic model. Since the objective is maximized, we modify the Hessian eigenvalues to enforce negative definiteness before computing the Newton-type direction. The resulting direction is used in a one-dimensional preference-based line search over $\eta\in[0,1]$. In the experiments, this line search is implemented using Thompson sampling over a fixed grid of 100 candidate step sizes.

At each iteration, we sample
$M=d$ local points from $B(\bm{x}_t,\epsilon)\cap\Theta$ and compare each of them against the current
iterate $\optvar_t$. The resulting pairwise observations are added to the data set, and the pairwise GP is
refit.

After the local comparisons, we compute
\[
    \bm{g}_t = \bm{\mu}^{\mathrm{pref}}_{\nabla f}(\optvar_t),
    \qquad
    \bm{H}_t = \bm{\mu}^{\mathrm{pref}}_{\nabla^2 f}(\optvar_t).
\]
These quantities define a local quadratic model. Since the latent utility is maximized, we require a
locally concave quadratic approximation before computing the Newton-type ascent direction. The
predicted Hessian is therefore symmetrized and modified by shifting or clipping its eigenvalues so
that the resulting matrix $\tilde H_t$ is negative definite. The search direction is then
\[
    \bm{p}_t = -\tilde{\bm{H}}_t^{-1} \bm{g}_t.
\]

To avoid always taking the full Newton-type step, PrefSQP performs a one-dimensional
preference-based line search over
\[
    \optvar_t(\eta)=\Pi_\Theta(\optvar_t+\eta \bm{p}_t), \qquad \eta\in[0,1].
\]
In our implementation, $[0,1]$ is discretized into a grid of candidate step sizes and Thompson
sampling is used to select the next point along this line. The selected point is then used as the next
iterate according to the preference-based line-search rule. The full procedure is summarized in
Algorithm~\ref{alg:prefsqp}.

\begin{algorithm}[t]
\caption{PrefSQP}
\label{alg:prefsqp}
\begin{algorithmic}[1]
\STATE Initialize data set \(\mathcal D_0\), current iterate \(\optvar_0\), local sampling radius \(\epsilon\), line-search resolution \(R\)
\FOR{\(t=0,\ldots,T-1\)}
    \STATE Draw local samples \(\{\optvar_{t,1},\ldots,\optvar_{t,M}\}\subset \mathcal B_\epsilon(\optvar_t)\cap\Theta\), with \(M=d\)
    \FOR{\(i=1,\ldots,M\)}
        \STATE Compare local sample \(\optvar_{t,i}\) against current iterate \(\optvar_t\)
        \STATE Observe ordered comparison \((\optvar^+_{t,i},\optvar^-_{t,i})\)
        \STATE Update \(\mathcal D_t \leftarrow \mathcal D_t \cup \{(\optvar^+_{t,i},\optvar^-_{t,i})\}\)
    \ENDFOR
    \STATE Fit pairwise GP model using \(\mathcal D_t\)
    \STATE Compute posterior mean gradient \(\bm{g}_t \leftarrow \bm{\mu}^{\mathrm{pref}}_{\nabla f}(\optvar_t)\)
    \STATE Compute posterior mean Hessian \(\bm{H}_t \leftarrow \bm{\mu}^{\mathrm{pref}}_{\nabla^2 f}(\optvar_t)\)
    \STATE Modify \(\bm{H}_t\) to obtain a negative definite matrix \(\widetilde{\bm{H}}_t\)
    \STATE Compute SQP direction \(\bm{p}_t \leftarrow -\widetilde{\bm{H}}_t^{-1} \bm{g}_t\)
    \STATE Construct line-search set
    \[
        \mathcal L_t
        =
        \left\{
        \Pi_\Theta(\optvar_t+\eta_r \bm{p}_t)
        \;:\;
        \eta_r \in \left\{0,\frac{1}{R-1},\ldots,1\right\}
        \right\}
    \]
    \STATE Select \(\optvar_{t+1}\in\mathcal L_t\) using Thompson sampling under the pairwise GP
    \STATE Set \(\mathcal D_{t+1}\leftarrow\mathcal D_t\)
\ENDFOR
\STATE \textbf{return} best observed incumbent \(\optvar^\star\)
\end{algorithmic}
\end{algorithm}

\subsection{TuRPBO: Trust-Region Preferential Bayesian Optimization}
\label{app:turpbo_full}
TuRPBO follows the trust-region adaptation rules of TuRBO, but replaces scalar improvement by
pairwise preference improvement. Let $\optvar_t^\star$ be the current incumbent and
$B_t=\{\optvar_{t,1},\ldots,\optvar_{t,q}\}$ the batch selected inside the trust region. An iteration is
successful if at least one candidate is preferred to the incumbent,
\[
    \exists \optvar_{t,i}\in B_t \quad \text{s.t.} \quad \optvar_{t,i} \succ \optvar_t^\star .
\]
In that case, the incumbent is updated to the preferred candidate, the success counter is
incremented, and the failure counter is reset. Otherwise, the incumbent remains unchanged, the
failure counter is incremented, and the success counter is reset. After $\tau_{\mathrm{succ}}$
consecutive successes, the trust-region length is doubled up to $L_{\max}$; after
$\tau_{\mathrm{fail}}$ consecutive failures, it is halved. If the length drops below $L_{\min}$,
the trust region is restarted.

\subsubsection{Full algorithm}
\begin{algorithm}[H]
\caption{TurPBO}
\label{alg:turpbo}
\begin{algorithmic}[1]
\STATE Initialize data set \(\mathcal D_0\), incumbent \(\optvar_0^\star\), trust-region length \(L_0=L_{\mathrm{init}}\), $\tau_{\mathrm{fail}}$ and $\tau_{\mathrm{succ}}$
\STATE Set \(c_{\mathrm{succ}}=0\), \(c_{\mathrm{fail}}=0\)
\FOR{\(t=0,\ldots,T-1\)}
    \STATE Fit pairwise GP model using \(\mathcal D_t\)
    \STATE Construct trust region \(\mathcal T_t\) around \(\optvar_t^\star\)
    \STATE Select batch \(B_t=\{\optvar_{t,1},\ldots,\optvar_{t,q}\}\subset \mathcal T_t\) using TS or qEUBO
    \STATE Compare candidates in \(B_t\) against incumbent \(\optvar_t^\star\)
    \IF{some \(\optvar_{t,i}\in B_t\) is preferred to \(\optvar_t^\star\)}
        \STATE Set \(\optvar_{t+1}^\star\) to the preferred candidate
        \STATE \(c_{\mathrm{succ}}\leftarrow c_{\mathrm{succ}}+1\), \(c_{\mathrm{fail}}\leftarrow 0\)
    \ELSE
        \STATE \(\optvar_{t+1}^\star\leftarrow \optvar_t^\star\)
        \STATE \(c_{\mathrm{fail}}\leftarrow c_{\mathrm{fail}}+1\), \(c_{\mathrm{succ}}\leftarrow 0\)
    \ENDIF
    \IF{\(c_{\mathrm{succ}}=\tau_{\mathrm{succ}}\)}
        \STATE \(L_{t+1}\leftarrow \min(2L_t,L_{\max})\)
        \STATE Reset \(c_{\mathrm{succ}}\) and \(c_{\mathrm{fail}}\)
    \ELSIF{\(c_{\mathrm{fail}}=\tau_{\mathrm{fail}}\)}
        \STATE \(L_{t+1}\leftarrow L_t/2\)
        \STATE Reset \(c_{\mathrm{succ}}\) and \(c_{\mathrm{fail}}\)
    \ELSE
        \STATE \(L_{t+1}\leftarrow L_t\)
    \ENDIF
    \IF{\(L_{t+1}<L_{\min}\)}
        \STATE Restart trust region around a new incumbent
    \ENDIF
    \STATE Update \(\mathcal D_{t+1}\) with the newly observed comparisons
\ENDFOR
\end{algorithmic}
\end{algorithm}

\clearpage
\section{Experimental Details and Additional Results}
\label{app:experiments}

\subsection{Shared Experimental Setup}
\paragraph{Evaluation and comparison budgets.}
We report budgets in terms of function evaluations, since these are the expensive operations in our target applications. In the pairwise setting, the number of comparisons is nevertheless almost identical to the number of function evaluations: after the two initial evaluations, every newly evaluated point is compared against an incumbent. Since function values are not observed by the algorithms, the incumbent is defined purely from preference data as the current winner under the observed comparisons. For most methods this is the best-seen point so far, updated whenever a newly evaluated point is preferred to the current incumbent. Thus, $N$ function evaluations yield $N-1$ pairwise comparisons. For the gradient-based methods, comparisons are made only during the local gradient-estimation phases against a local incumbent; the subsequent center updates are not immediately evaluated. Hence, if $G$ gradient steps are taken, $N$ function evaluations yield $N-G$ pairwise comparisons. Thus, a budget such as $2+10d$ refers to function evaluations, not acquisition steps, but corresponds to approximately the same number of pairwise comparisons.

\begin{table}[h]
    \centering
    \small
    \caption{Experiment-level settings used across benchmark experiments. Values marked ``experiment-specific'' are reported in the corresponding benchmark sections.}
    \label{tab:experiment_hyperparameters}
    \renewcommand{\arraystretch}{1.15}
    \begin{tabular}{p{6.2cm} p{6.6cm}}
        \toprule
        \textbf{Setting} & \textbf{Value / Description} \\
        \midrule
        Number of independent runs & 10 \\
        Initial design size for synthetic and GP benchmarks & 2 random points \\
        Initial design size for policy-search benchmarks & $5d$ random evaluations \\
        Optimization budget for GP benchmarks & $2 + 10d$ evaluations \\
        Optimization budget for synthetic benchmarks & $2 + 10d$ evaluations \\
        Optimization budget for policy-search benchmarks & $5d + 10d$ evaluations \\
        Input normalization & All benchmark domains mapped to $[0,1]^d$ \\
        Observation type given to algorithms & Noisy pairwise comparisons only \\
        Noise model & Gaussian noise with standard deviation equal to $10\%$ of the approximate function-value range \\
        Reported metric & Best-seen value $f_{\mathrm{best}}$ and cumulative performance $\sum_k f_k$ \\
        Uncertainty visualization & Median with 25/75 percentiles, unless stated otherwise \\
        Hyperparameter learning & Experiment-specific; known for within-model GP experiments and learned for out-of-model/synthetic/policy experiments \\
        Baselines & Sobol, GLISp, qEUBO, HB-EI \\
        Local methods & GIPBO, GIPBOr, PrefSQP, TurPBO \\
        \bottomrule
    \end{tabular}
\end{table}

\subsection{Shared PairwiseGP Settings}
\begin{table}[h]
    \centering
    \small
    \caption{Shared surrogate-model and acquisition-optimization settings for PairwiseGP-based methods. These settings apply to EUBO, GIPBO, GIPBOr, PrefSQP, and TuRPBO unless stated otherwise.}
    \label{tab:shared_pairwisegp_hyperparameters}
    \renewcommand{\arraystretch}{1.15}
    \begin{tabular}{p{6.2cm} p{6.6cm}}
        \toprule
        \textbf{Hyperparameter} & \textbf{Value / Description} \\
        \midrule
        GP model & BoTorch \texttt{PairwiseGP} \\
        Kernel & ScaleKernel with squared-exponential base kernel \\
        Input domain & Unit hypercube $[0,1]^d$ \\
        Lengthscale constraint & $[0.05, 0.5]$ \\
        Preference noise convention & BoTorch convention with implicit $\sigma=1$ and adjusted output scale \\
        Acquisition optimization restarts & 5 \\
        Raw samples for acquisition optimization & 1024 \\
        Batch size $q$ & 1 unless stated otherwise \\
        \bottomrule
    \end{tabular}
\end{table}

Note that the \texttt{PairwiseGP} implementation in BoTorch, used for local PBO methods and EUBO, 
implicitly assumes that the observation noise $\sigma$ equals $1$. Different noise levels are 
therefore represented through the kernel output scale. We follow this convention when computing 
preference likelihoods for the conditional expectation in the adapted GI acquisition function.
\FloatBarrier
\newpage
\subsection{Algorithm Configurations}
\begin{table}[h]
    \centering
    \small
    \caption{Hyperparameters and implementation choices for GLISp.}
    \label{tab:glisp_hyperparameters}
    \renewcommand{\arraystretch}{1.15}
    \begin{tabular}{p{6.2cm} p{6.6cm}}
        \toprule
        \textbf{Hyperparameter} & \textbf{Value / Description} \\
        \midrule
        Comparison tolerance & $10^{-4}$ \\
        Initial random points for virtual initialization & 10 \\
        Lower bound & $0$ in each dimension \\
        Upper bound & $1$ in each dimension \\
        Input domain & Unit hypercube $[0,1]^d$ \\
        \bottomrule
    \end{tabular}
\end{table}

\begin{table}[h]
    \centering
    \small
    \caption{Hyperparameters and implementation choices for GIPBO and GIPBOr.}
    \label{tab:gipbo_hyperparameters}
    \renewcommand{\arraystretch}{1.15}
    \begin{tabular}{p{6.2cm} p{6.6cm}}
        \toprule
        \textbf{Hyperparameter} & \textbf{Value / Description} \\
        \midrule
        Gradient step length $\eta$ & 0.5 \\
        Minimum absolute step length & 0.05 \\
        Number of local queries per gradient estimate $M$ & $d$ \\
        Number of gradient steps $N$ & 10 \\
        Acquisition search bounds $(\delta_{\min}, \delta_{\max})$ & $[0.05, 0.2]$ \\
        Resolution for $d>2$ & 50 \\
        Hypersphere radii for $d>2$ & $[0.05, 0.1, 0.15, 0.2]$ \\
        Virtual preference likelihood threshold & 0.1 \\
        Gradient normalization & Mahalanobis norm induced by SE lengthscales \\
        GIPBOr acquisition selection & Random samples within the same local search bounds \\
        \bottomrule
    \end{tabular}
\end{table}

\begin{table}[h]
    \centering
    \small
    \caption{Hyperparameters and implementation choices for PrefSQP.}
    \label{tab:prefsqp_hyperparameters}
    \renewcommand{\arraystretch}{1.15}
    \begin{tabular}{p{6.2cm} p{6.6cm}}
        \toprule
        \textbf{Hyperparameter} & \textbf{Value / Description} \\
        \midrule
        Local sampling radius $\epsilon$ & 0.05 \\
        Number of local samples per SQP step & $d$ \\
        Hessian modification & Eigenvalue modification to enforce negative definiteness \\
        Line-search domain & $\eta \in [0,1]$ \\
        Line-search resolution & 100 \\
        Line-search acquisition & Thompson sampling unless stated otherwise \\
        Step direction & Newton-type direction from posterior mean gradient and Hessian \\
        \bottomrule
    \end{tabular}
\end{table}

\begin{table}[h]
    \centering
    \small
    \caption{Hyperparameters and implementation choices for TuRPBO. Here, $d$ denotes the problem dimensionality and $q$ the batch size.}
    \label{tab:turpbo_hyperparameters}
    \renewcommand{\arraystretch}{1.15}
    \begin{tabular}{p{6.2cm} p{6.6cm}}
        \toprule
        \textbf{Hyperparameter} & \textbf{Value / Description} \\
        \midrule
        Batch size $q$ & 1 \\
        Acquisition function & Thompson sampling unless stated otherwise \\
        Acquisition optimization restarts & 5 \\
        Raw samples for acquisition optimization & 1024 \\
        Initial trust-region length $L_{\mathrm{init}}$ & 0.8 \\
        Minimum trust-region length $L_{\min}$ & $0.5^7$ \\
        Maximum trust-region length $L_{\max}$ & 1.6 \\
        Success tolerance $\tau_{\mathrm{succ}}$ & $d$ \\
        Failure tolerance $\tau_{\mathrm{fail}}$ & $\lceil \max(4/q, d/q) \rceil$ \\
        Trust-region expansion & $L \leftarrow \min(2L, L_{\max})$ \\
        Trust-region contraction & $L \leftarrow L/2$ \\
        Restart criterion & Restart if $L < L_{\min}$ \\
        \bottomrule
    \end{tabular}
\end{table}
\clearpage
\subsection{GP Sample-Path Experiments}
\label{app:gp_experiments}
In addition to the plots in Figure \ref{fig:gp_comp}, that only shows the final performances for the different algorithms in the within and out-of-model comparison, Figure \ref{fig:within-performance-curves} and \ref{fig:out-of-performance-curves} show the full performance curves for these experiments. 
In the experiments, we sampled a new GP sample for each run with a random initialization. All algorithms used the same initial points per run. 

\begin{figure}[h]
    \centering
    \includegraphics[width=\linewidth]{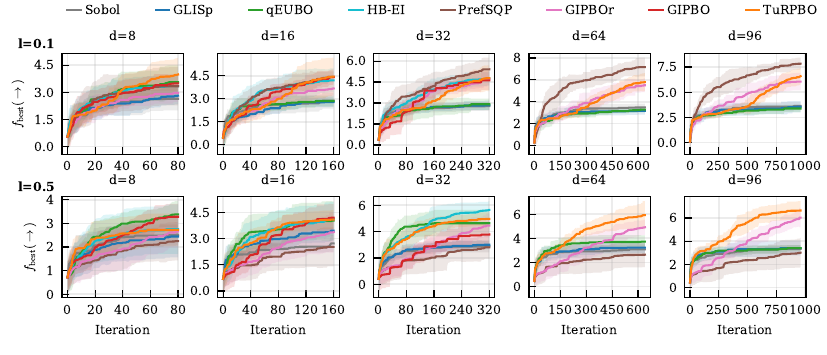}
    \caption{Within-model comparison. Performance over iterations for GP samples between $8$ and $96d$.}
    \label{fig:within-performance-curves}
\end{figure}

\begin{figure}[h]
    \centering
    \includegraphics[width=\linewidth]{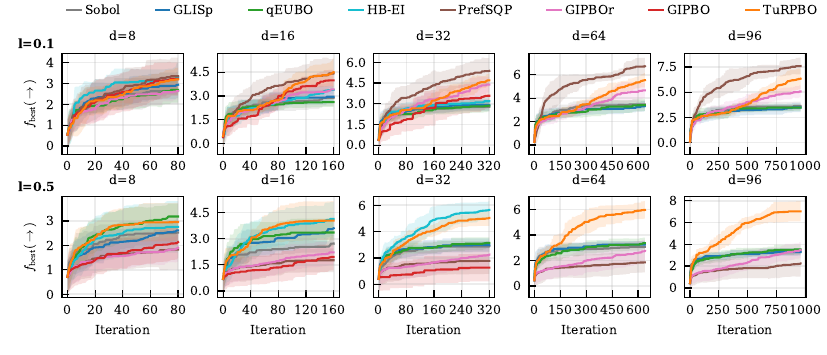}
    \caption{Out-of-model comparison. Performance over iterations for GP samples between $8$ and $96d$.}
    \label{fig:out-of-performance-curves}
\end{figure}

GIPBOr performs worse than full GIPBO in 8 and 16 dimensions, especially for short lengthscales. In 32 dimensions, this advantage largely vanishes. This is caused by the implementation cap on the number of evaluated acquisition candidates: evaluating the full preferential GI acquisition becomes expensive, and the capped candidate set is too sparse to reliably approximate the acquisition maximizer in higher dimensions.

Comparing the within and out-of model comparison, especially GIPBO and GIPBOr degrade in performance, when the model is not well-specified. This can be explained by the fact that the gradient step is scaled using the lengthscale as described in \ref{sec:GIPBO_implementation_details}.

Furthermore, in the within-model comparison, the HB-EI method performs competitively also for short lengthscales. Demonstrating the advantages of the more computationally expensive skew GP model.

\subsection{Standard Benchmark Details}
For the synthetic benchmark experiments, we use the test functions from the BoTorch library \cite{balandat2020botorch}. All benchmark domains are normalized to $[0,1]^d$ and all algorithms are initialized with the same two random points in each run. The optimization budget is $2+10d$ evaluations. The algorithms observe only noisy pairwise comparisons, where the noise standard deviation is set to $10\%$ of the approximate function-value range. 
The main synthetic benchmark results are shown in Figure~\ref{fig:synthetic} and the compute times are reported in \ref{tab:compute-time}.

As in the results we only report mean $\pm$ stdv in our plots for better visibility, we report the results again in more detail and perform a statistical significance test. We use a pairwise Wilcoxon signed-rank test \cite{wilcoxon1945} to compare our results similar to the test in \cite{erarslan_consecutive_2025}. As we perform multiple tests, we apply a Holm correction to the results. We always test if the best performing algorithm (bold in the plot) performs significantly different than the others. Algorithms which are not significantly performing differently than the best performing one are underlined in the plot. 

\begin{figure}[h]
    \centering
    \includegraphics[width=0.9\linewidth]{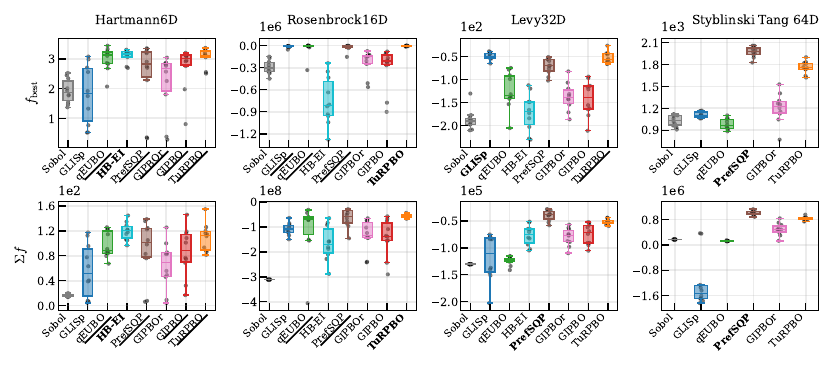}
    \caption{Final performance boxplots for the different optimization functions. The best performing algorithm has a bold label. The ones that that are significantly different have a normal label, the ones that are not statistically different are underlined.}
    \label{fig:sythetic_signifance}
\end{figure}

\subsection{Policy-Search}
\label{app:policy_details}

For the policy search task we use the implementation from ARS \cite{horia2018simple}. 
We warmstart the experiment by selecting the best point out of $5d$ iterations and comparing it to a random value to emulate the case of a promising starting point. In Table~\ref{tab:benchmark-final-cumulative-best-only-statistics}, we report the results and use a paired Wilcoxon as in the previous section, to determine, if the best performing algorithm outperforms the baselines statistically significant. The bold value in the table is the best algorithm, the underlined value are not significantly different.   

\begin{table}[h]
\centering
\caption{Final and cumulative benchmark performance in the policy search task with warmstarting, the bold label is the best for each metric, the underlined values are not statistically significantly difference in the statistical test. }
\resizebox{\textwidth}{!}{%
\begin{tabular}{llcccccccc}
\toprule
Metric & Function & Sobol & GLISp & qEUBO & HB-EI & PrefSQP & GIPBOr & GIPBO & TuRPBO \\
\midrule
$f_{\mathrm{best}} \uparrow$ &
HopperLinearPolicyTask 33D &
360 $\pm$ 239 &
630 $\pm$ 346 &
629 $\pm$ 301 &
372 $\pm$ 246 &
427 $\pm$ 239 &
453 $\pm$ 330 &
489 $\pm$ 318 &
768 $\pm$ 305 \\
$f_{\mathrm{best}} \uparrow$ &
Walker2dLinearPolicyTask 102D &
113 $\pm$ 39 &
112 $\pm$ 47 &
122 $\pm$ 52 &
-- &
\underline{202 $\pm$ 83} &
\underline{267 $\pm$ 131} &
-- &
\textbf{275 $\pm$ 109} \\
$\sum_k f_{k} \uparrow$ &
HopperLinearPolicyTask 33D &
4888 $\pm$ 519 &
17457 $\pm$ 7249 &
7324 $\pm$ 2582 &
\underline{70445 $\pm$ 47721} &
\underline{81884 $\pm$ 76833} &
\underline{82143 $\pm$ 59092} &
\underline{90767 $\pm$ 66906} &
\textbf{91796 $\pm$ 53199} \\
$\sum_k f_{k} \uparrow$ &
Walker2dLinearPolicyTask 102D &
-6580 $\pm$ 237 &
-5921 $\pm$ 622 &
-6975 $\pm$ 180 &
-- &
\underline{31895 $\pm$ 16962} &
\textbf{40389 $\pm$ 25571} &
-- &
\underline{37517 $\pm$ 18229} \\
\bottomrule
\end{tabular}
}
\label{tab:benchmark-final-cumulative-best-only-statistics}
\end{table}

In addition to the warmstarted RL-experiment, we also performed the experiment starting from an initialization with one comparison consisting of two random data points. 
\begin{figure}[h]
    \centering
    \includegraphics[width=\linewidth]{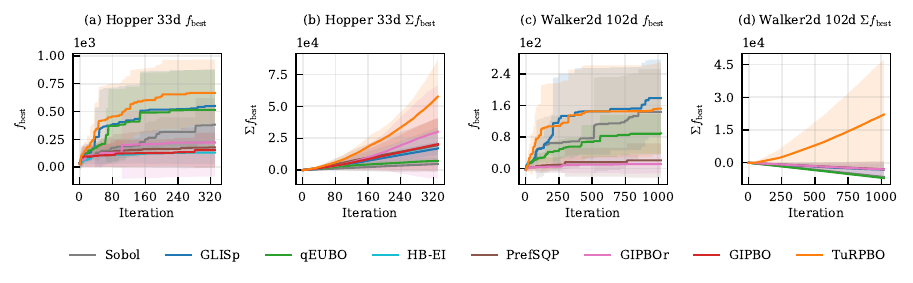}
    \caption{Mujoco Benchmark without warmstarting}
    \label{fig:rl_without_warmstart}
\end{figure}

Starting from a bad initialization leads to a performance degrade of the gradient based-search methods, that perform even worse than random search. Nevertheless, in terms of cumulative performance, they still perform better or equally to GLISp which has a higher best observed performance. 

TuRBO performs consistently well in the experiments, however, the final performance is not better than random search in the hopper example. 

\subsection{Ablations}
\label{app:ablations}
For the different algorithm variants, we perform an ablation in $32d$. We conduct the experiments in an out-of-model comparison on $32d$ GP samples with lengthscales $0.1$ and $0.5$, on the Levy function in $32d$ as well as the Styblinsky-Tang $32d$ test function. 

For PrefSQP, we ablate the different variants with different acquisition options for the gradient: using the suggested step length (fixed, $\eta = 1$), using line search with EUBO or with thompson sampling. Futhermore, we investigate if a Newton step or the uncertainty aware subproblem from \cite{brunzema_bayesqp_2026} should be used (EV). Figure~\ref{fig:PrefSQP-ablation} shows that in the synthetic setting, all variants perform similarly. On the GP functions, the fixed Newton step works best.  
\begin{figure}[h]
    \centering
    \includegraphics[width=\linewidth]{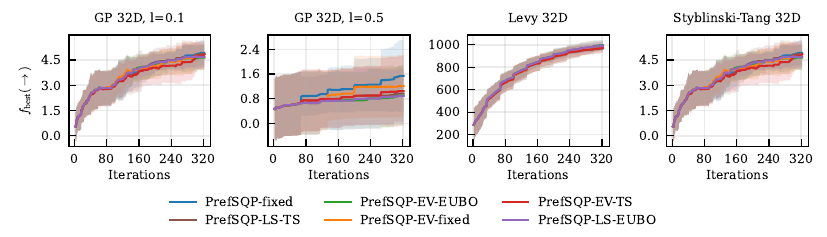}
    \caption{PrefSQP variant ablation}
    \label{fig:PrefSQP-ablation}
\end{figure}

For TuRPBO, we also compare different acquisition strategies. We compare the perturbed Thompson sampling variant, which is our base TuRPBO version, to expected improvement and qEUBO \cite{astudillo2023qEUBO}. On the GP samples TuRPBO performs best, while on the standard benchmarking functions, the qEUBO variant performs equivalently well.  
\begin{figure}[h]
    \centering
    \includegraphics[width=\linewidth]{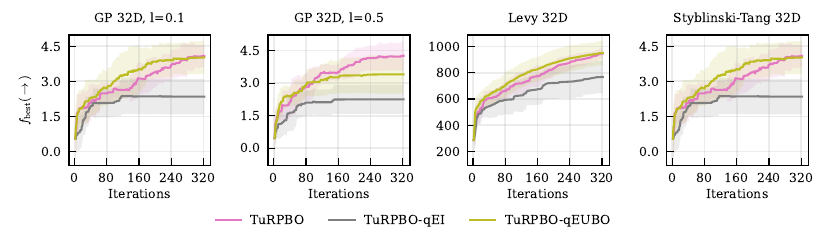}
    \caption{Different Acquisition strategies in TuRPBO}
    \label{fig:placeholder}
\end{figure}

In all our settings, we compare against qEUBO with batch size one. However, a key selling point of qEUBO is, that the batch size can be increased and therefore the number of comparisons can be reduced. In this paper, we focus on the number of experiments, as we assume that in a policy searching task the environment interactions are more expensive than the comparisons. Nevertheless, we also compare the performance of different qEUBO batch sizes as a reference in Figure~\ref{fig:batch_qeubo}. 
\begin{figure}[h]
    \centering
    \includegraphics[width=\linewidth]{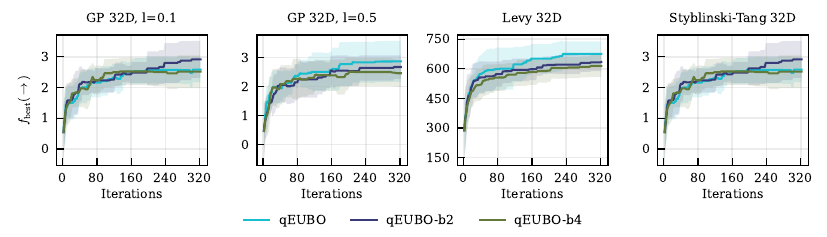}
    \caption{Comparison of qEUBO batch size 1, 2 and 4}
    \label{fig:batch_qeubo}
\end{figure}
In our benchmarking setup, all qEUBO variants perform equally over the number of experiments, however batch size two would only need half the number of decisions and batch size four only a fourth. However, the final performance after the number of iterations in the $32d$ case is still lower than all in the local variants we propose. 

Finally, we ablate different noise level in the same setting to evaluate the robustness of the algorithms towards noise and plot the results in Figure~\ref{fig:noise}.

\begin{figure}[h]
    \centering
    \includegraphics[width=\linewidth]{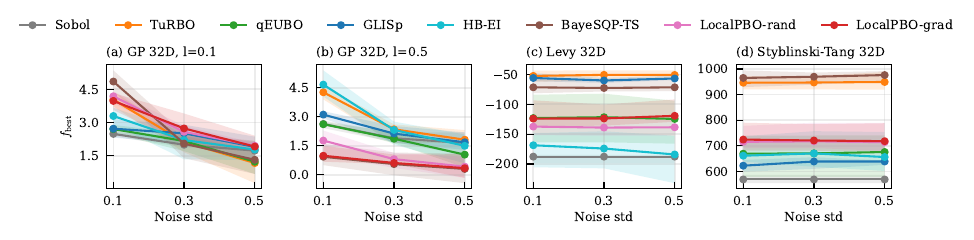}
    \caption{Noise Ablation}
    \label{fig:noise}
\end{figure}

\newpage
\section{Appendix D: Computational Resources} \label{app:resources}
All experiments were conducted using CPUs only on a high-performance computing (HPC) cluster. At the time of the experiments, the cluster ran \textit{Rocky Linux 9.6}. Jobs were submitted via Slurm to the \texttt{c23mm} partition. Each run was executed on a single node with one Slurm task and eight CPU cores allocated to that task, with \texttt{5000M} of memory requested per CPU core. 

The compute time for the main results are reported as average per run in table \ref{tab:compute-time2}.

\begin{table}[H]
\centering
\caption{Compute time per algorithm across benchmark functions per run ($10d$ iterations) with mean and standard deviation over all 10 runs for each experiment.}
\resizebox{\textwidth}{!}{%
\begin{tabular}{lcccccccc}
\toprule
 & \multicolumn{4}{c}{Global methods} & \multicolumn{4}{c}{Local methods} \\
\cmidrule(lr){2-5}\cmidrule(lr){6-9}
Function & Sobol & GLISp & qEUBO ($q=1$) & HB-EI & GIPBO & GIPBOr & TuRBO & BayeSQP \\
\midrule
{[within] GPSample 8D ls=0.1} & 0.039 $\pm$ 0.003 & 15 $\pm$ 9 & 14 $\pm$ 10 & 107 $\pm$ 9 & 311 $\pm$ 24 & 1.3 $\pm$ 1.0 & 45 $\pm$ 3 & 6.7 $\pm$ 1.8 \\
\addlinespace
{[within] GPSample 8D ls=0.5} & 0.037 $\pm$ 0.001 & 6.0 $\pm$ 1.8 & 19 $\pm$ 5 & 287 $\pm$ 31 & 532 $\pm$ 232 & 1.1 $\pm$ 0.4 & 47 $\pm$ 3 & 12 $\pm$ 9 \\
\addlinespace
{[within] GPSample 16D ls=0.1} & 0.099 $\pm$ 0.022 & 12 $\pm$ 2 & 75 $\pm$ 26 & 626 $\pm$ 70 & 1794 $\pm$ 105 & 1.6 $\pm$ 0.5 & 101 $\pm$ 14 & 21 $\pm$ 4 \\
\addlinespace
{[within] GPSample 16D ls=0.5} & 0.097 $\pm$ 0.022 & 12 $\pm$ 2 & 96 $\pm$ 30 & 1206 $\pm$ 218 & 2900 $\pm$ 472 & 2.1 $\pm$ 0.6 & 103 $\pm$ 10 & 30 $\pm$ 15 \\
\addlinespace
{[within] GPSample 32D ls=0.1} & 0.42 $\pm$ 0.65 & 30 $\pm$ 5 & 341 $\pm$ 132 & 6294 $\pm$ 494 & 28955 $\pm$ 15083 & 3.5 $\pm$ 0.8 & 345 $\pm$ 125 & 117 $\pm$ 50 \\
\addlinespace
{[within] GPSample 32D ls=0.5} & 0.39 $\pm$ 0.63 & 30 $\pm$ 7 & 193 $\pm$ 64 & 6852 $\pm$ 981 & 34879 $\pm$ 6996 & 3.9 $\pm$ 0.7 & 252 $\pm$ 31 & 147 $\pm$ 51 \\
\addlinespace
{[within] GPSample 64D ls=0.1} & 0.49 $\pm$ 0.16 & 91 $\pm$ 10 & 13171 $\pm$ 3467 & -- & -- & 12 $\pm$ 4 & 1564 $\pm$ 604 & 839 $\pm$ 172 \\
\addlinespace
{[within] GPSample 64D ls=0.5} & 0.49 $\pm$ 0.07 & 94 $\pm$ 19 & 1707 $\pm$ 1100 & -- & -- & 17 $\pm$ 3 & 1195 $\pm$ 229 & 1583 $\pm$ 665 \\
\addlinespace
{[within] GPSample 96D ls=0.1} & 0.93 $\pm$ 0.40 & 227 $\pm$ 17 & 27048 $\pm$ 10433 & -- & -- & 28 $\pm$ 4 & 5902 $\pm$ 3369 & 3710 $\pm$ 295 \\
\addlinespace
{[outof] GPSample 8D ls=0.1} & 0.042 $\pm$ 0.003 & 10 $\pm$ 7 & 40 $\pm$ 6 & 304 $\pm$ 19 & 368 $\pm$ 40 & 3.1 $\pm$ 1.1 & 58 $\pm$ 4 & 23 $\pm$ 6 \\
\addlinespace
{[outof] GPSample 8D ls=0.5} & 0.044 $\pm$ 0.007 & 8.4 $\pm$ 2.7 & 42 $\pm$ 8 & 285 $\pm$ 23 & 432 $\pm$ 25 & 2.9 $\pm$ 0.6 & 56 $\pm$ 8 & 22 $\pm$ 3 \\
\addlinespace
{[outof] GPSample 16D ls=0.1} & 0.18 $\pm$ 0.26 & 13 $\pm$ 2 & 420 $\pm$ 673 & 1419 $\pm$ 233 & 1939 $\pm$ 161 & 10.0 $\pm$ 18.6 & 303 $\pm$ 527 & 309 $\pm$ 759 \\
\addlinespace
{[outof] GPSample 16D ls=0.5} & 0.25 $\pm$ 0.48 & 11 $\pm$ 3 & 191 $\pm$ 98 & 1304 $\pm$ 217 & 1984 $\pm$ 242 & 4.0 $\pm$ 1.8 & 135 $\pm$ 18 & 61 $\pm$ 15 \\
\addlinespace
{[outof] GPSample 32D ls=0.1} & 0.22 $\pm$ 0.06 & 29 $\pm$ 6 & 775 $\pm$ 367 & 6787 $\pm$ 500 & 13568 $\pm$ 1851 & 7.0 $\pm$ 0.9 & 489 $\pm$ 88 & 310 $\pm$ 54 \\
\addlinespace
{[outof] GPSample 32D ls=0.5} & 0.20 $\pm$ 0.03 & 30 $\pm$ 6 & 943 $\pm$ 599 & 7015 $\pm$ 437 & 14496 $\pm$ 3794 & 8.1 $\pm$ 1.2 & 445 $\pm$ 82 & 264 $\pm$ 54 \\
\addlinespace
{[outof] GPSample 64D ls=0.1} & 0.87 $\pm$ 0.73 & 87 $\pm$ 12 & 3210 $\pm$ 2094 & -- & -- & 36 $\pm$ 10 & 2682 $\pm$ 694 & 2650 $\pm$ 266 \\
\addlinespace
{[outof] GPSample 64D ls=0.5} & 1.0 $\pm$ 0.8 & 91 $\pm$ 14 & 3873 $\pm$ 1609 & -- & -- & 49 $\pm$ 27 & 2683 $\pm$ 368 & 2422 $\pm$ 339 \\
\addlinespace
{[outof] GPSample 96D ls=0.1} & 0.93 $\pm$ 0.21 & 243 $\pm$ 25 & 12120 $\pm$ 3617 & -- & -- & 109 $\pm$ 54 & 11096 $\pm$ 1346 & 12789 $\pm$ 1605 \\
\addlinespace
{[outof] GPSample 96D ls=0.5} & 0.78 $\pm$ 0.19 & 242 $\pm$ 27 & 16148 $\pm$ 11031 & -- & -- & 100 $\pm$ 22 & 11050 $\pm$ 2367 & 11315 $\pm$ 2092 \\
\addlinespace
{[within] GPSample 96D ls=0.5} & 1.7 $\pm$ 1.4 & 237 $\pm$ 23 & 11749 $\pm$ 6302 & -- & -- & 45 $\pm$ 5 & 4142 $\pm$ 603 & 7727 $\pm$ 3790 \\
\addlinespace
{[rl] Hopper} & 3.3 $\pm$ 0.2 & 29 $\pm$ 3 & 1418 $\pm$ 993 & 7082 $\pm$ 629 & 15894 $\pm$ 2471 & 35 $\pm$ 15 & 489 $\pm$ 73 & 351 $\pm$ 114 \\
\addlinespace
{[rl] Walker2d} & 11 $\pm$ 1 & 287 $\pm$ 27 & 22652 $\pm$ 17705 & -- & -- & 164 $\pm$ 42 & 17039 $\pm$ 7493 & 15262 $\pm$ 7326 \\
\bottomrule
\end{tabular}
}
\label{tab:compute-time2}
\end{table}

The used software packages are listed in \ref{tab:software}.

\begin{table}[h]
\centering
\caption{Software packages used in the benchmark environment.}
\begin{tabular}{ll}
\toprule
\textbf{Package} & \textbf{Version} \\
\midrule
Python        & 3.12 \\
\addlinespace
NumPy         & 1.26.4 \\
SciPy         & 1.12.0 \\
pandas        & 3.0.2 \\
scikit-learn  & 1.8.0 \\
numba         & 0.65.1 \\
Cython        & 3.2.4 \\
Matplotlib    & 3.10.9 \\
\addlinespace
PyTorch       & 2.11.0 (CPU) \\
GPyTorch      & 1.14 \\
linear-operator & 0.6 \\
BoTorch       & 0.15.1 \\
GPy           & 1.13.2 \\
paramz        & 0.9.6 \\
\addlinespace
CVXOPT        & 1.3.3 \\
GLIS          & 2.0.2 \\
pyDOE         & 0.9.3 \\
pyswarm       & 0.8.0 \\
pyro-ppl      & 1.9.1 \\
\addlinespace
MuJoCo        & 3.8.0 \\
Gymnasium     & 1.3.0 \\
\addlinespace
tqdm          & 4.67.3 \\
\bottomrule
\end{tabular}
\label{tab:software}
\end{table}

%\clearpage
%\input{checklist.tex}

\end{document}